\documentclass[12pt,a4paper,dvipsnames]{article}%
\usepackage[american]{babel}
\usepackage{xfrac}
\usepackage{bbm}
\usepackage{authblk}
\usepackage{amsthm}
\newtheorem{proposition}{Proposition}[section]
\usepackage[format=plain,font=it]{caption}

\usepackage{geometry}
\usepackage{helvet} 

\usepackage[round]{natbib} 
\usepackage{mathtools} 
\usepackage{booktabs} 
\usepackage{tikz} 

\usepackage{float}
\usepackage{subfig}
\usepackage{amsmath}
\usepackage{amsfonts}
\usepackage{comment}
\usepackage{dsfont}

\DeclareMathOperator{\g}{g}

\title{From Uncertainty to Precision:\\Enhancing Binary Classifier Performance\\through Calibration\thanks{The authors would like to thank Philipp Ratz for his valuable comments. Emmanuel Flachaire and Ewen Gallic acknowledge that the project leading to this publication has received funding from the French government under the ``France 2030'' investment plan managed by the French National Research Agency (reference: ANR-17-EURE-0020) and from Excellence Initiative of Aix-Marseille University -- A*MIDEX. Fran\c{c}ois Hu acknowledges that the project is funded by the Natural Sciences and Engineering Research Council of Canada (NSERC) Emerging Infectious Diseases Modelling Initiative (EIDM), awarded to the Mathematics for Public Health (MfPH) program.}}

\usepackage{graphicx,pstricks,psfrag}
\definecolor{bleu}{RGB}{0,101,189}
\definecolor{vert}{HTML}{004D40}
\definecolor{rose}{HTML}{D81B60}
\definecolor{bleuTOL}{HTML}{332288}
\definecolor{wongBlack}{RGB}{0,0,0}
\definecolor{wongGold}{RGB}{230, 159, 0}
\definecolor{wongLightBlue}{RGB}{86, 180, 233}
\definecolor{wongGreen}{RGB}{0, 158, 115}
\definecolor{wongYellow}{RGB}{240, 228, 66}
\definecolor{wongBlue}{RGB}{0, 114, 178}
\definecolor{wongOrange}{RGB}{213, 94, 0}
\definecolor{wongPurple}{RGB}{204, 121, 167}
\definecolor{colUncalibrated}{RGB}{191, 191, 191}
\definecolor{colRecalibrated}{RGB}{197, 214, 231}

\definecolor{bleuTOL}{HTML}{332288}
\definecolor{vertTOL}{HTML}{117733}
\definecolor{vertClairTOL}{HTML}{44AA99}
\definecolor{bleuClairTOL}{HTML}{88CCEE}
\definecolor{sableTOL}{HTML}{DDCC77}
\definecolor{parmeTOL}{HTML}{CC6677}
\definecolor{magentaTOL}{HTML}{AA4499}
\definecolor{roseTOL}{HTML}{882255}

\usepackage{hyperref}
\hypersetup{
    plainpages=false,
    bookmarksopen,
    bookmarksnumbered,
    unicode=true,          
    pdftoolbar=true,        
    pdfmenubar=true,        
    pdffitwindow=false,     
    pdfstartview={FitH},    
    pdftitle={Frm Uncertainty to Precision: Enhancing Binary Classifier Performance through Calibration},    
    pdfauthor={},     
    pdfkeywords={calibration, local regression, machine-learning}, 
    pdfnewwindow=true,      
    colorlinks=true,       
    linkcolor=rose,          
    citecolor={wongBlue},        
    filecolor={rose},      
    urlcolor={rose},           
}

\author[1]{Agathe~Fernandes~Machado\thanks{Corresponding author: \href{mailto:fernandes_machado.agathe@courrier.uqam.ca}{fernandes\_machado.agathe@courrier.uqam.ca}}}
\author[1]{Arthur~Charpentier}
\author[2]{Emmanuel~Flachaire}
\author[2]{Ewen~Gallic}
\author[3]{Fran\c cois Hu}

\affil[1]{%
    \footnotesize Département de Mathématiques\\
    Université du Québec à Montréal\\
    Montréal, Québec, Canada
}
\affil[2]{%
    \footnotesize Aix Marseille Univ, CNRS, AMSE\\
    Marseille, France
}
\affil[3]{%
    \footnotesize Université de Montréal\\
    Montréal, Québec, Canada
}

\usepackage[misc]{ifsym}
\makeatletter
\def\@fnsymbol#1{%
   \ifcase#1\or
   \TextOrMath ~ \dagger\or
   \TextOrMath {\footnotesize\Letter} \dagger\or
   \TextOrMath \textdaggerdbl \ddagger \or
   \TextOrMath \textsection  \mathsection\or
   \TextOrMath \textparagraph \mathparagraph\or
   \TextOrMath \textbardbl \|\or
   \TextOrMath {\textdagger\textdagger}{\dagger\dagger}\or
   \TextOrMath {\textdaggerdbl\textdaggerdbl}{\ddagger\ddagger}\else
   \@ctrerr \fi
}
\makeatother

\usepackage{fancyhdr} 
\newcommand{\authornames}{\footnotesize\textsc{Fernandes Machado, Charpentier, Flachaire, Gallic, Hu}}
\usepackage{etoolbox}
\makeatletter
\patchcmd{\NAT@test}{\else \NAT@nm}{\else \NAT@nmfmt{\NAT@nm}}{}{}

\DeclareRobustCommand\citepos
  {\begingroup
   \let\NAT@nmfmt\NAT@posfmt
   \NAT@swafalse\let\NAT@ctype\z@\NAT@partrue
   \@ifstar{\NAT@fulltrue\NAT@citetp}{\NAT@fullfalse\NAT@citetp}}

\let\NAT@orig@nmfmt\NAT@nmfmt
\def\NAT@posfmt#1{\NAT@orig@nmfmt{#1's}}
\makeatother

\begin{document}

\maketitle
\thispagestyle{empty}

\begin{abstract}
  The assessment of binary classifier performance traditionally centers on discriminative ability using metrics, such as accuracy. However, these metrics often disregard the model's inherent uncertainty, especially when dealing with sensitive decision-making domains, such as finance or healthcare. Given that model-predicted scores are commonly seen as event probabilities, calibration is crucial for accurate interpretation. In our study, we analyze the sensitivity of various calibration measures to score distortions and introduce a refined metric, the Local Calibration Score. 
  Comparing recalibration methods, we advocate for local regressions, emphasizing their dual role as effective recalibration tools and facilitators of smoother visualizations. We apply these findings in a real-world scenario using Random Forest classifier and regressor to predict credit default while simultaneously measuring calibration during performance optimization.
\end{abstract}


\section{Introduction}\label{sec:intro}

Binary classification tasks are prevalent in learning algorithms, as diverse scenarios require binary decisions. Examples include predicting default risk or accident occurrence in insurance or finance as well as disease likelihood in healthcare.
To improve reliability, particularly in sensitive decision-making contexts, a classifier must possess strong discriminatory capabilities. 
Typically, classifiers are trained to optimize goodness-of-fit criteria, often based on the accuracy of class predictions. However, goodness-of-fit criteria, such as accuracy or AUC, do not consider the varying confidence levels assigned by the algorithm to each prediction. If the sole objective is effective class prediction, then the classifier fulfills its purpose. Nevertheless, there are instances where interest extends beyond the predicted class to the associated likelihood.
This occurs when predicting loan repayment defaults \citep{PD-RF} or accident incidences, as risk transfer pricing is usually tied directly to event probabilities. In such cases, the model-predicted scores of classifiers are often interpreted as event probabilities. 
Yet, in these examples, achieving accurate interpretation as probabilities necessitates effective calibration of the model.
A well-calibrated model ensures precise understanding of the predicted scores as probabilities. For instance, if a model assigns a predicted probability of 80\% to events, the observed proportion of those events occurring over the long run --according to \citepos{dawid1982well} terminology-- should ideally align with the predicted value of 80\%.

Simple classifiers such as Logistic Regression models typically exhibit overall calibration \citep{Mildenhall1999ASR} due to their design in the empirical risk minimization problem, but evaluating their local calibration conditioned on predicted score values poses challenges \citep{kull2017EJS}. In addition, when using more opaque models, such as Random Forest (RF) or Neural Networks, the interpretability of calibration becomes more nuanced, with differing views on their potential (mis)calibration \citep{ICMLcalibration2005, pmlr-v70-guo17a, stackedRF2020, krishnan2020improving, minderer2021revisiting}. 
Consequently, various metrics and post-processing recalibration methods, including Platt scaling \citep{platt1999probabilistic}, isotonic regression \citep{zadrozny2002transforming}, and Beta calibration \citep{kull2017EJS}, have been proposed to measure and correct poor calibration in classification models, particularly in scenarios requiring confidence scores.

Considering the comprehensive studies on the concept of calibration, determining the most appropriate metric or recalibration method for a specific dataset and its trained algorithm is not straightforward. Various calibration metrics diverge significantly in their evaluation of models, and a consensus has yet to be established, even in binary tasks. In this regard, after presenting numerical and graphical tools to measure the calibration of a binary predictive model in Section~\ref{sec:calibration}, we propose simulating a synthetic dataset for which the true distribution of probabilities is known. By deliberately manipulating these probabilities, we create distorted versions to emulate uncalibrated scores that might be produced by an ML model. 
With access to the true distribution, we precisely measure the miscalibration of the distorted probabilities using the Mean Squared Error (MSE). Consequently, we can identify calibration metrics that closely resemble the MSE, and introduce a novel metric, the Local Calibration Score (LCS), which uses local regression techniques. Subsequently, in Section~\ref{sec:recalibration}, we present recalibration approaches to address poor calibration of distorted probabilities. Despite the intentional miscalibration of these scores, the traditional performance metrics have not deteriorated.

The analysis of synthetic data yields insights favoring the novel calibration metric LCS. Subsequently, we compute this measure in a real-world scenario discussed in Section~\ref{sec:expe-RF}, in which a Random Forest algorithm is employed to predict the risk of default. We train both an RF classifier and regressor on the dataset, with the latter demonstrating superior accuracy and calibration, contrasting the results reported by \cite{calibRF2008}. While selecting the RF algorithm and optimizing its hyperparameters using data from \cite{misc_default_of_credit_card_clients_350}, we recreate a scenario in which decision makers prioritize finely tuned goodness-of-fit metrics. During this procedure, caution is advised to avoid compromising calibration for the sake of discriminative capacity, particularly with the aim of using model predicted scores.

Our contributions can be summarized as
follows:
\begin{itemize}
    \item In the case of binary regression within our Data Generating Process (DGP), enabling exact calibration calculation, the Expected Calibration Error (ECE) does not emerge as the most robust calibration metric, as well as the Brier score.
    \item Based on visualization techniques that involve local polynomial regression for calibration curves, we introduce a novel calibration metric named LCS, the relevance of which is validated through the assessment of ground-truth miscalibration on synthetic data.
    \item When observing the progression of the novel calibration metric --LCS-- for different AUC levels during the optimization of both RF regressor and classifier algorithms, we highlight that integrating a calibration metric in the optimization process is significant if one intends to utilize the scores predicted by the classifier.
\end{itemize}

\section{Calibration}
\label{sec:calibration}

Consider a binary variable $D$ that takes the value 1 if an event occurs and 0 otherwise. In this context, the probability of the event depends on individual characteristics, \textit{i.e.}, $p_i = s(\mathbf{x}_i)$, where, with sample size $n > 0$, $i=1,\ldots, n$ represents individuals, and $\mathbf{x}_i$ the characteristics. The goal is to estimate this probability using a model, such as a Generalized Linear Model (GLM) or an ML model such as an RF. These models estimate a score $\hat{s}(\mathbf{x}_i) \in [0,1]$, allowing the classification of observations based on the estimated probability of the event. By setting a probability threshold $\tau$ in 
 $[0, 1]$, one can predict the class of each observation: 1 if the event occurs, and 0 otherwise. However, to interprete the score as a probability, it is crucial that the model is well-calibrated. For a binary variable $D$, a model is well-calibrated when \citep{definition-calibration}
\begin{equation}
\mathbb{P}(D = 1 \mid \hat{s}(\mathbf{x}) = p) = p, \quad \forall p \in [0,1]\enspace,\label{eq-calib}
\end{equation}
which is equivalent to:
\begin{equation}
\mathbb{E}[D \mid \hat{s}(\mathbf{x}) = p] = p, \quad \forall p \in [0,1]\enspace.\label{eq-calib-E}
\end{equation}
For example, if $D$ represents a credit default, with $D=1$ indicating default and $D=0$ indicating no default, a model predicting credit default is well-calibrated if the values of model's predicted probability $\hat{s}(\mathbf{x})$ closely match the actual observed probability of default. This means that if the model predicts a probability of 0.8 for a particular individual, the actual default rate for individuals with a predicted probability of 0.8 should be close to 0.8.

It should be mentioned that conditioning by $\{\hat{s}(\mathbf{x})=p\}$ leads to the concept of (local) calibration; however, as discussed by \cite{bai2021don}, $\{\hat{s}(\mathbf{x})=p\}$ is \textit{a.s.} a null mass event. Thus, calibration should be understood in the sense that 
$$
\mathbb{E}[D \mid \hat{s}(\mathbf{x}) = p]\overset{a.s.}{\to} p\text{ when }n\to\infty\enspace,
$$
\sloppy meaning that, asymptotically, the model is well-calibrated, or locally well-calibrated in $p$, for any $p$. From the dominated convergence theorem, it also signifies that, ``on average,'' the model is well-calibrated. To account for this, we will consider multiple replications of finite samples in the simulation study in Section~\ref{subsec:poor-calibration}.

\subsection{Measuring Calibration}
\label{sec:measure-calibration}

To assess the calibration of a predictive model, the literature provides both metrics and visual approaches for evaluating $\hat{s}$. Given the continuous nature of the score with a null mass event, various methods have been proposed to detect (\mbox{mis-})calibration. We will briefly introduce the most popular measures before presenting a new calibration metric (refer to the LCS in Section~\ref{subsubsec:locfitcurve}).

\subsubsection{Quantile based measures}

\paragraph{Calibration curve}

In the binary case, based on Equation~\ref{eq-calib-E}, constructing a calibration curve to visualize the calibration of a model involves estimating the function $\g(\cdot)$ that measures miscalibration on its predicted scores $\hat{s}(\mathbf{x})$:
\begin{equation}
\g :
\begin{cases}
[0,1] \rightarrow [0,1]\\
p \mapsto \g(p) := \mathbb{E}[D \mid \hat{s}(\mathbf{x}) = p]
\end{cases}\enspace.
\end{equation}
The $\g$ function for a well-calibrated model is the identity function $\g(p) = p$. In practice, from this real-valued setting, it is challenging to have a sufficient number of observations in the training dataset with identical scores to effectively measure calibration defined in Equation~\ref{eq-calib}, resulting in a lack of robustness in the estimation process of these probabilities. A common method for estimating calibration is to group observations into $B$ bins, defined by the empirical quantiles of predicted values $\hat{s}(\mathbf{x})$. The average of observed values, denoted $\bar{d}_b$ with $b\in \{1, \ldots, B\}$, in each bin $b$ can then be compared with the central value of the bin. Thus, a calibration curve can be constructed by plotting the centers of each bin on the x-axis and the averages of corresponding observations on the y-axis, also referred to as reliability diagrams \citep{reliability-diagrams-weather1990}. When the model is well-calibrated, all $B$ points lie on the bisector. 

\paragraph{Expected Calibration Error}

Given a sample size $n$, the Expected Calibration Error (ECE) \citep{naeini2015obtaining} is determined using two metrics within each bin $b\in\{1, \ldots, B\}$ of quantile-binned predicted scores $\hat{s}(\mathbf{x})$: accuracy $\text{acc}(b)$, which measures the average of empirical probabilities or fractions of correctly predicted classes, and confidence $\text{conf}(b)$, indicating the model's average confidence within bin $b$ by averaging predicted scores. The ECE is then computed as the average over the bins using:
\[
\text{ECE} = \sum_{b=1}^{B} \frac{n_b}{n} \mid \text{acc}(b) - \text{conf}(b) \mid
\]
where $n_b$ is the number of observations in bin $b$. Given that each bin $b$ is associated with set $\mathcal{I}_b$ containing the indices of instances within that bin,
$$
\text{acc}(b) = \frac{1}{n_b} \sum_{i \in \mathcal{I}_b} \mathds{1}_{\hat{d}_i = d_i} \quad \text{and}\quad \text{conf}(b) = \frac{1}{n_b} \sum_{i \in \mathcal{I}_b} \hat{s}(\mathbf{x}_i)\enspace,
$$
are, respectively, the accuracy and the confidence of the model in bin $b$. The predicted class $\hat{d}_i$ for observation $i$ is determined based on a classification threshold $\tau\in [0,1]$, where $\hat{d}_i = 1$ if $\hat{s}(\mathbf{x}_i) \geq \tau$ and $0$ otherwise.

Notably, the aforementioned ECE corresponds to the exact definition of calibration for multi-class prediction \citep{pmlr-v70-guo17a}.  For a recent study investigating the application of this measure to geographic data with ordinal classes, we refer to \citep{machado2024geospatial}.

\subsubsection{Local Regression based measure}
\label{subsubsec:locfitcurve}

We propose an alternative approach to visualize 
model calibration, aiming for a smoother representation than that provided by the method based on quantiles. 

\paragraph{Smoothed calibration curve}

Instead of defining bins, we estimate calibration using a local regression. The benefit of employing local regression techniques over bin-based visualization lies in the fact that given the number of data points, the precision of quantile binning can be suboptimal when determining the appropriate bin count. By contrast, with local regression, one can specify the percentage of nearest neighbors, providing greater flexibility.

Local regression involves applying local polynomial regression techniques, with the definition of ``local'' adopting various forms, including bandwidth or the count of nearest neighbors. 
These characteristics are determined using different approaches, enhancing the precision and smoothness of the analysis. The degree of the polynomial determines the construction of the polynomial regression. For this purpose, we will use the \texttt{locfit} package from R \citep{loader1999}. The \texttt{locfit} library selects a specific set of evaluation points from the dataset to conduct the regression fit. 
By default, the evaluation structure follows a tree-based pattern: the algorithm confines the dataset within a rectangle and divides it into two equal-sized rectangles. Polynomial parameters from the fit at selected points are then used to interpolate values at other locations. While this approach may have been disregarded in high dimensions due to poor properties, it is highly efficient in small dimensions, as in this case with only one predictive feature, $\hat{s}(\boldsymbol{x})$. Moreover, this approach is backed by a solid theoretical foundation, with well-established statistical properties.

To obtain a calibration curve, we fit a local regression model with a degree of 0.\footnote{The choice of degree 0 allows us to compute the local mean of observed events in the vicinity of a predicted probability, illustrating calibration from Equation \ref{eq-calib-E}.} We perform a regression of observed events on predicted probabilities. Then, we employ the trained model to make predictions on a linear space with values in $[0, 1]$. This approach provides a continuous calibration profile, based on uniformly distributed points to evaluate $\g$.

The local regression model relies on neighboring points to make predictions. However, problems arise at the edges when the range of predicted scores $\hat{s}(\mathbf{x})$ does not cover the entire interval $[0,1]$. In such cases, predicted values beyond this range may deviate from the bisector, leading to a misinterpretation of calibration. To prevent this issue, we adjust the linear space used for predictions by the local regression model to align with the full range of observed scores $\hat{s}(\mathbf{x})$.

\paragraph{Local Calibration Score}
Similar to how the Integrated Calibration Index is defined using the LOESS regression method \citep{Austin2019TheIC}, we introduce our methodology, the Local Calibration Score (LCS), which is based on the calibration curve constructed using \texttt{locfit}, as detailed in Section \ref{subsubsec:locfitcurve}. The calculation of the LCS relies on the disparities between this curve and the bisector, in the range from 0 to 1, weighted by the density of the predicted scores $\hat{s}(\mathbf{x})$. To execute this, following the methodology outlined in Section \ref{subsubsec:locfitcurve}, a local regression of degree 0, denoted as $\hat{\g}$, is fitted to the predicted scores $\hat{s}(\mathbf{x})$. This fit is then applied to a vector of linearly spaced values within the interval $[0,1]$. Each of these points is denoted by $l_i$, where $i \in \{1, \ldots, N\}$, with $N$ being the target number of points on the visualization curve. The LCS is calculated by averaging the squared differences between each predicted score $\hat{\g}(l_i)$ and its corresponding linearly spaced value $l_i$, weighted by the density of the observed scores at $l_i$, corresponding to $w_i$:
\begin{equation}
    \text{LCS} = \sum_{i=1}^{n}w_i \big(\hat{\g}(l_i) - l_i\big)^2\enspace .\label{eq-locfit-score}
\end{equation}

\subsection{Impact of a Poor Calibration}
\label{subsec:poor-calibration}

In this section, we aim to explore the impact of a poorly-calibrated model.\footnote{Replication material:
\href{https://github.com/fer-agathe/calibration_binary_classifier}{https://github.com/fer-agathe/calibration\_binary\_classifier}.} To do this, we create synthetic datasets where we know the true probability of a binary event occurrence in advance. Instead of relying on model-derived scores $\hat{s}(\mathbf{x})$, we apply transformations directly to the probabilities. These transformed values serve as surrogate scores, representing what could be obtained from an ML model. This approach enables us to assess the effects of poor calibration on different calibration metrics, particularly our novel calibration measure LCS, as well as on various calibration curves. Additionally, it allows us to investigate the impact of poor calibration on standard goodness-of-fit metrics.

\paragraph{Synthetic data} Following \citet{gutman2022propensity}, we consider a binary variable assumed to follow a Bernoulli distribution: $D_i\sim B(p_i)$, where $p_i$ is the probability of observing $D_i = 1$. We define $p_i$ using the sigmoid function:%
\begin{equation}
p_i = \frac{1}{1+\exp(-\eta_i)}\enspace,\label{eq-true-propensity}
\end{equation}
Here, $\eta_i$ is defined by the equation:%
\begin{equation}
\eta_i = a_1 x_1 + a_2 x_2 + a_3 x_3 - a_4 x_4 + \varepsilon_i\enspace, \label{eq-propensity-eta}
\end{equation}
\sloppy where $x_1$, $x_2$, $x_3$, and $x_4$ are randomly drawn from a uniform distribution $\mathcal{U}[0,1]$, where $(a_1, \dots, a_4)$ are scalars arbitrarily set in our case to $\left(a_1, a_2, a_3, a_4\right) = \left(0.1, 0.05, 0.2, -0.05\right)$, and where $\varepsilon_i \sim \mathcal{N}(0, 0.5^2)$.
We generate $n=2,000$ observations using this DGP. 

Of particular importance, with this experimental setup and a well-defined framework, we establish in the following proposition that logistic regression is asymptotically well-calibrated. The proof is detailed in Appendix~\ref{proof:prop}.

\begin{proposition}\label{prop:logistic}
Consider a dataset $\{(d_i,\mathbf{x{_i}})\}$, where 
{$\mathbf{x}$}
are $k$ features ($k$ being fixed), so that $D|\boldsymbol{X}=\mathbf{x} \sim \mathcal{B}\big(s(\mathbf{x})\big)$ where
$$
s(\mathbf{x})=\frac{\exp[\beta_0+\mathbf{x}^\top\boldsymbol{\beta}]}{1+\exp[\beta_0+\mathbf{x}^\top\boldsymbol{\beta}]}.
$$
Let $\widehat{\beta}_0$ and $\widehat{\boldsymbol{\beta}}$ denote maximum likelihood estimators. Then, for any $\mathbf{x}$, the score is defined as 
$$
\hat{s}(\mathbf{x})=\frac{\exp[\hat\beta_0+\mathbf{x}^\top\hat{\boldsymbol{\beta}}]}{1+\exp[\hat\beta_0+\mathbf{x}^\top\hat{\boldsymbol{\beta}}]}
$$
is well-calibrated in the sense that
$$
\mathbb{E}[D \mid \hat{s}(\mathbf{x}) = p]\overset{a.s.}{\to} p\text{ as }n\to\infty.
$$
\end{proposition}
\begin{proof}
    Proof in Appendix~\ref{proof:prop}. If $k$ were increasing with $n$, \cite{bai2021don} showed that logistic regression is over-confident. However, assuming here that $k$ is fixed provides a complementary perspective.
\end{proof}
This justifies the possibility of achieving a perfectly calibrated model with this DGP, making the proposed synthetic data appealing to study (mis-)calibration.

Next, we introduce two types of transformations to the true probabilities $p$ to simulate uncalibrated modeling: one directly applied to the latent probability and another applied to the linear predictor $\eta$. Specifically, we introduce a scaling parameter that modifies the latent probability, altering Equation~\ref{eq-true-propensity} to:%
\begin{equation}
p_i^{u} = \left(\frac{1}{1+\exp(-\eta_i)}\right)^{\alpha}\enspace. \label{eq-true-propensity-decalibrated}
\end{equation}

A second scaling parameter which modifies the linear predictor changes Equation~\ref{eq-propensity-eta} to:%
\begin{equation}
\eta_i^u = \gamma \times (-0.1)x_1 + 0.05x_2 + 0.2x_3 - 0.05x_4 + \varepsilon_i, \label{eq-propensity-eta-decalibrated}
\end{equation}

The resulting values, given by $p^u$, are considered as the scores $\hat{s}(\mathbf{x})$ that could be returned by a predictive model. Note that when $\alpha=\gamma=1$, no transformation occurs, representing the benchmark situation of a well-calibrated model. In the simulations, we examine variations in $\alpha$ and $\gamma$ across the range ${1/3, 1, 3}$, considering each parameter individually while keeping the other fixed at $1$. The effects of the transformations on the probabilities are shown in Figure~\ref{fig-probas-transfo}.\footnote{See Figure~\ref{fig-distrib-probas-transfo} in the Appendix for histograms.} Notably, decreasing (increasing) $\alpha$ shifts values closer to 1 (to 0). Decreasing $\gamma$ concentrates values around 0.5, while increasing $\gamma$ disperses probabilities around 0.5.

\begin{figure}[ht]
\centering
\includegraphics[width = \linewidth]{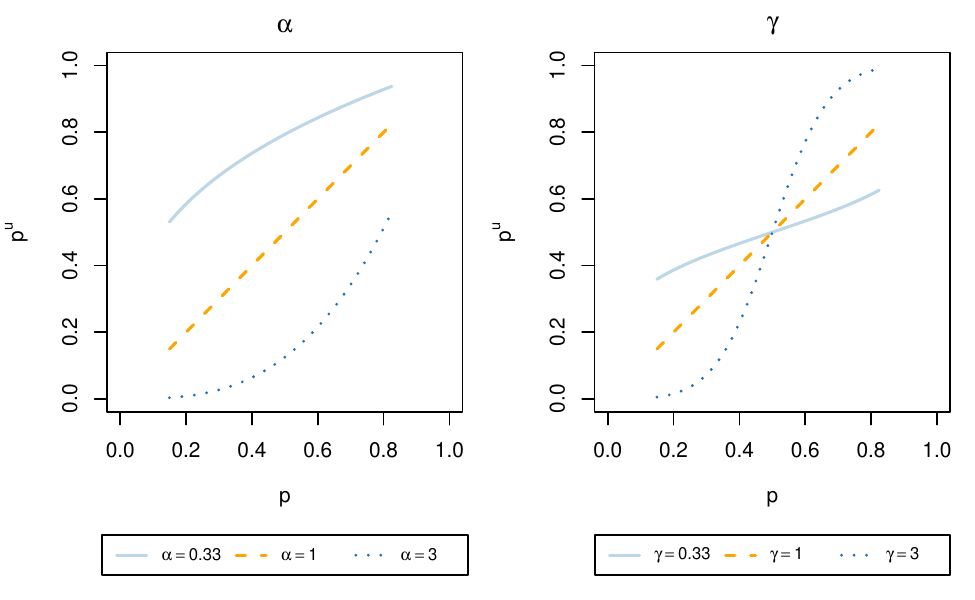}
\caption{Distorted Probabilities as a Function of True Probabilities, Depending on the Value of $\alpha$ (left) or $\gamma$ (right).}\label{fig-probas-transfo}
\end{figure}

For each value of parameter $\alpha$ or $\gamma$, we generate $200$ independent samples, each consisting of $2,000$ observations. Within each sample, after applying $p^u$, we compute the various metrics previously mentioned. 
Since we are aware of the true scores, we calculate the associated MSE, $n^{-1} \sum_{i=1}^{n} (p_i - p_i^u)^2$, representing the ground-truth calibration and enabling us to understand the characteristics that accurate calibration metrics should exhibit. Its equivalent when the true probability distribution is not known corresponds to the Brier score \citep{brier_1950}, where $p_i$ is replaced by $d_i$. 
Figure~\ref{fig-calib-simuls-metrics} presents the results in the form of boxplots for the calculation of the calibration measures when $\alpha$ varies (top) or when $\gamma$ varies (bottom). For each metric, a degradation is observed when the scores deviate from the true probabilities, \textit{i.e.}, when $\alpha \ne 1$ or $\gamma \ne 1$. It is noteworthy that, with data generated from our DGP, the proposed LCS approach remarkably reflects the true error being close to 0 when $\alpha = \gamma = 1$.
While the Brier score mirrors the dynamics of the MSE, it registers excessively high values, surpassing 0.23, for the true probability distribution. Conversely, the ECE appears less suitable for assessing calibration within this binary scenario.

\begin{figure}[ht]
\centering
\includegraphics[width = \linewidth]{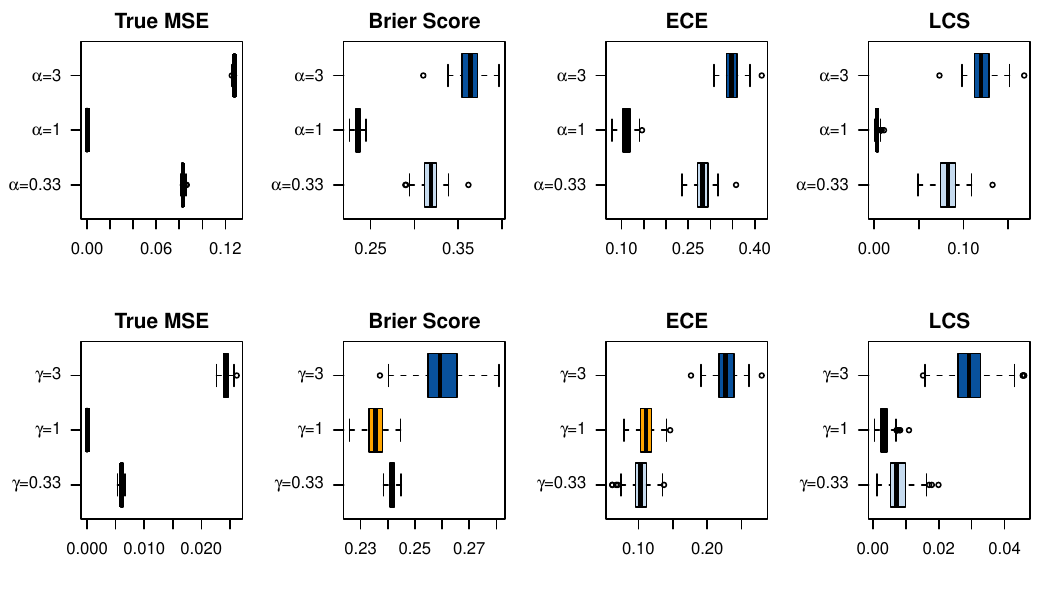}
\caption{Calibration Metrics on 200 Simulations for each Value of $\alpha$ (top) or $\gamma$ (bottom).}\label{fig-calib-simuls-metrics}
\end{figure}

Calibration can also be assessed using calibration curves. Figure~\ref{fig-calib-simuls-calib-curve-locfit} shows for each type of score distortion the average of the calibration curve computed using local regression with \texttt{locfit}, for the 200 simulations. The error band corresponds to a 95\% bootstrap confidence interval. The distribution of true probabilities is depicted at the top of each graph.\footnote{This is represented as a histogram obtained by averaging the number of observations in each bin across the 200 simulations.} The curves for cases where the model is effectively well-calibrated ($\alpha=1$ or $\gamma=1$) are notably aligned with the bisector. Conversely, in other scenarios, the calibration curves deviate considerably from the alignment with the bisector. Although similar visualizations are obtained with the calibration curves computed using bins defined by quantiles (see Figure~\ref{fig-calib-simuls-calib-curve-quant} in the Appendix), calibration curves obtained with local regression exhibit smoother patterns.

\begin{figure}[ht]
\centering
\includegraphics[width = .7\linewidth]{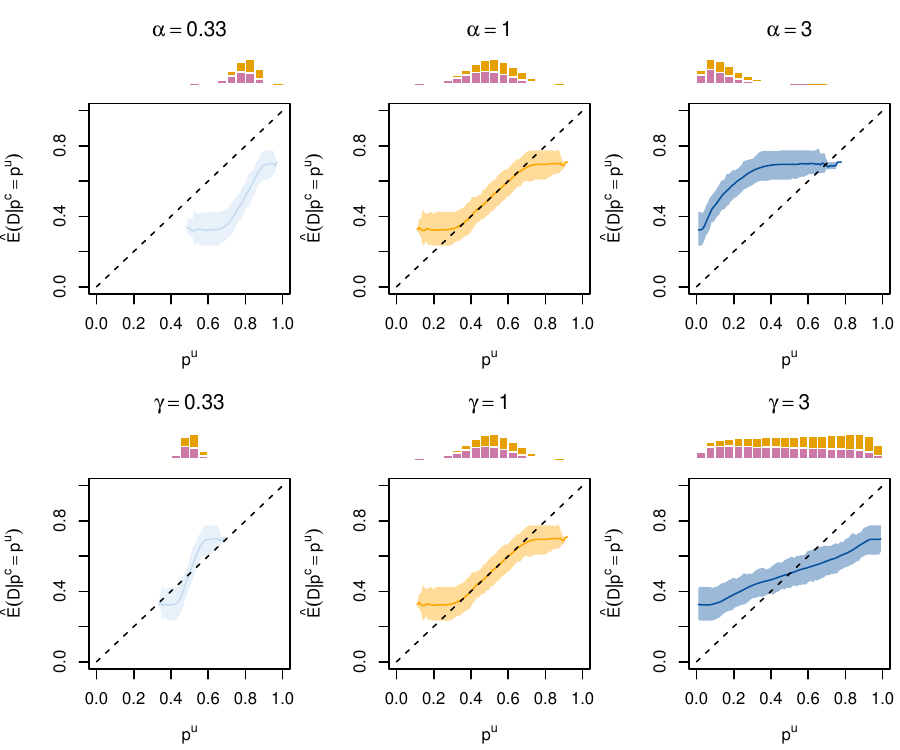}
\caption{Calibration Curve Obtained with Local Regression, on 200 simulations for each Value of $\alpha$ (top) or $\gamma$ (bottom). Distribution of the true probabilities are shown in the histograms (gold for $d=1$, purple for $d=0$).}\label{fig-calib-simuls-calib-curve-locfit}
\end{figure}

Finally, we explore the sensitivity of standard goodness-of-fit metrics
to calibration. The applied transformations to probabilities for defining scores reflecting a poorly calibrated model show no discernible impact on standard goodness-of-fit measures, as illustrated in Figure~\ref{fig-calib-simuls-std-metrics}. Metrics such as accuracy, sensitivity, specificity, or AUC are not impacted by monotone transformations on the probabilities. When the focus shifts to using predicted scores from a binary classifier rather than merely class prediction accuracy, standard goodness-of-fit metrics may not detect poor calibration. This section explores the benefits of the proposed DGP, highlighting the rationale for employing a post-hoc recalibrator to address this issue.

In the following section, we examine different recalibration approaches, drawing from various techniques in the literature. We then delve deeper into our understanding of local regression methodologies for recalibration in Section~\ref{subsubsec:local_reg}.

\begin{figure}[ht]
\centering
\includegraphics[width = \linewidth]{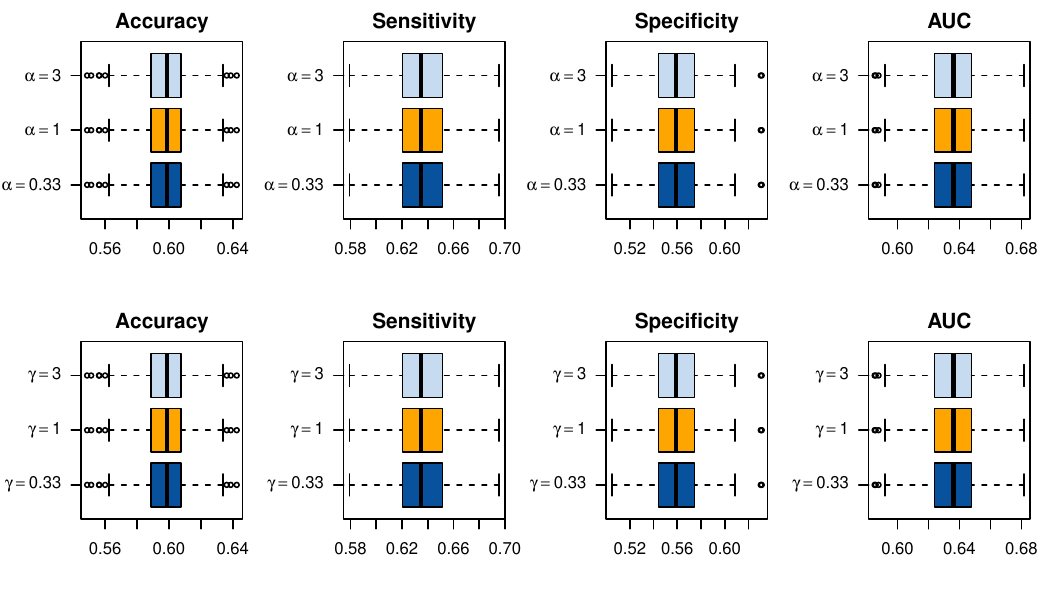}
\caption{Standard Goodness of Fit Metrics on 200 Simulations for each Value of $\alpha$ (top) or $\gamma$ (bottom). The probability threshold is set to $\tau=0.5$.}\label{fig-calib-simuls-std-metrics}
\end{figure}

\section{Recalibration}
\label{sec:recalibration}

When using scores generated by a model predicting the probability of a binary event, the literature advocates recalibrating the model by applying a transformation to the scores \citep{platt1999probabilistic, zadrozny2002transforming, kull2017EJS}.

To address overfitting during the learning of this mapping, it is recommended to partition data into three sets \citep{zadrozny2002transforming, kull2017EJS}: a training set for classifier training, a calibration set for recalibrator training, and a test set for computing calibration metrics. Adequate data is required for this tripartite division.

In this section, we present the prevalent recalibration techniques found in the literature. Then, we evaluate their impact on calibration performance using metrics and visualization methods from Section \ref{sec:calibration}.

\subsection{Recalibration Methods}\label{sec-recalib-methods}

\paragraph{Platt Scaling}

This approach was initially introduced to map SVM outputs to well-calibrated posterior probabilities \citep{platt1999probabilistic}. The method consists of applying Logistic Regression to the output probabilities of a binary classifier. The Logistic Regression parameters, $a$ and $b$, are learned on a calibration dataset, held out from a train dataset:
\begin{equation}
    \g\big(\hat{s}(\mathbf{x})\big) = \frac{1}{1+\exp{\left\{-\big(a\hat{s}(\mathbf{x})+b\big)\right\}}}\enspace.\label{eq-plattformula}
\end{equation}

\paragraph{Isotonic Regression}

This solution arises from a constrained optimization problem \citep{zadrozny2002transforming}, solved using the Pool-Adjacent-Violators Algorithm, ensuring that corrected predicted scores remain monotonic. In Equation \ref{eq-isoreg}, $d_{(i)}$ corresponds to the value in $\{d_1,\cdots,d_n\}$ associated with the $i$-th largest predicted score $\{\hat{s}(\mathbf{x}_1),\cdots,\hat{s}(\mathbf{x}_n)\}$.

\begin{equation}
    \begin{array}{rrclcl}
    \displaystyle \min_{\beta_1,\ldots, \beta_n} & \multicolumn{3}{l}{\sum_{i=1}^n (d_{(i)} - \beta_i)^2}\\
    \textrm{s.t.} & \beta_1 \leq \ldots \leq \beta_n
    \label{eq-isoreg}
\end{array}\enspace .
\end{equation}
Isotonic regression assumes the initial model's predicted scores $\hat{s}(\mathbf{x})$ are well-ordered, limiting its ability to correct non-monotonic probability distortions, as seen in methods like Platt scaling (Equation~\ref{eq-plattformula}).

\paragraph{Beta Calibration}

The Beta calibration method \citep{kull2017EJS} estimates the following calibration curve using three parameters, $a$, $b$ and $c$:
\begin{equation}
    \g\big(\hat{s}(\mathbf{x})\big) = \frac{1}{1+\exp{\left\{-c\right\}}\big(\sfrac{\hat{s}(\mathbf{x})^a}{(1-\hat{s}(\mathbf{x}))^b}\big)}\enspace.
\end{equation}

The condition $a$, $b \geq 0$ leads to an increasing map function. In contrast to Platt scaling, the Beta calibration family includes the identity function, allowing it to maintain score calibration when it is already calibrated.

\subsubsection{Local Regression}
\label{subsubsec:local_reg}

Interestingly, the local regression method serves a dual role, functioning both as a visualization tool (as discussed in Section~\ref{subsubsec:locfitcurve}) and as an approach for recalibration. Specifically, when fitting a local regression of degree 0 to the predicted scores $\hat{s}(\mathbf{x})$, it estimates the following quantity: $\hat{\mathbb{E}}[D \mid \hat{s}(\mathbf{x}) \in \mathcal{V}_p]$, where $p \in [0,1]$. Here, $\mathcal{V}_p$ represents the neighborhood of a given $p$, defined using a percentage of nearest neighbors among the set of evaluation points when using \texttt{locfit} (as mentioned in Section~\ref{subsubsec:locfitcurve}).
By definition of local regression of degree 0, the estimation of the expected value approximates the function $\g$ defined in Equation~\ref{eq-calib-E}. Nevertheless, to improve the smoothness of expectancy estimations, particularly in regions with limited data points (e.g., when the predicted scores $\hat{s}(\mathbf{x})$ are close to 0 and 1), employing \texttt{locfit} with polynomial degrees of 1 and 2 is an alternative worth considering, while remaining aware of the potential occurrence of negative values.

\subsection{Scores Recalibration}

We apply the recalibration techniques presented in Section~\ref{sec-recalib-methods} to our simulated datasets. For both  calibration and test samples, we calculate calibration and goodness-of-fit metrics after replacing the uncalibrated scores $p^u$ with their recalibrated values, denoted $p^c$.
For each metric and scenario where we transformed the true probabilities to induce poor calibration, we compute the difference between the metric obtained with the recalibrated scores $p^c$ and that obtained with the uncalibrated scores $p^u$. The results for MSE, AUC and LCS are presented in Figure~\ref{fig-recalib-simuls-metrics-gamma3}, specifically for $\gamma=3$.\footnote{Further values for $\alpha$ and $\gamma$ are in the Appendix, Figure \ref{fig-recalib-simuls-calib-metrics}.}

Regardless of the method employed, recalibration leads to a decrease in MSE in both the calibration and test samples. However, this appear to occur at the expense of a reduction in AUC in the test sample. Meanwhile, calibration, measured with the LCS, indeed improves after applying any of the considered techniques. The same observation holds true for the calibration curves calculated using the method based on local regression.\footnote{Calibration curves for all values of $\gamma$ and $\alpha$ are provided in the Appendices, in Figures~\ref{fig-recalib-simuls-calib-curve-quant-true_prob-no-calib}, \ref{fig-recalib-simuls-calib-curve-quant-platt-isotonic}, and~\ref{fig-recalib-simuls-calib-curve-quant-locfit} for curves obtained by quantile binning, and in Figures ~\ref{fig-recalib-simuls-calib-curve-locfit-true_prob-no-calib}, \ref{fig-recalib-simuls-calib-curve-locfit-platt-isotonic-beta}, and~\ref{fig-recalib-simuls-calib-curve-locfit-locfit} for curves obtained by local regression.}

\begin{figure}[ht]
\centering
\includegraphics[width = \linewidth]{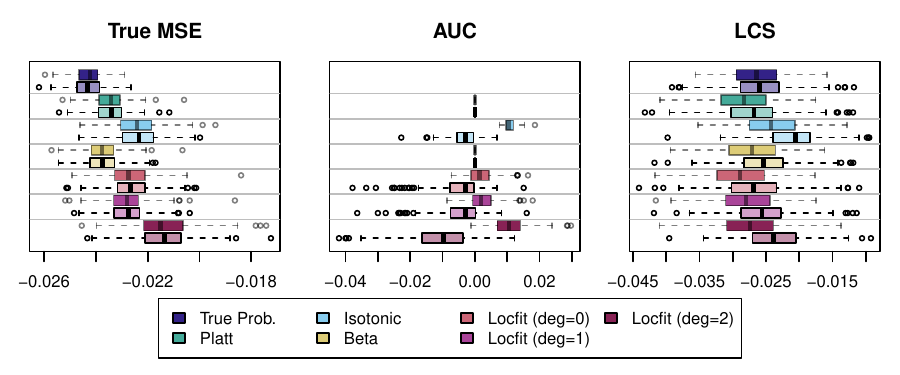}
\caption{Metrics After Recalibration (for $\gamma=3$), on the Calibration (transparent colors) and on the Test Set (full colors).}\label{fig-recalib-simuls-metrics-gamma3}
\end{figure}

Overall, we observe that across our 200 simulations, the use of recalibration techniques allows for achieving improved calibration, which comes with sacrificing a slight decrease of precision. However, caution must be exercised when interpreting these results. The subsequent section seeks to extend the analysis by using a real-world dataset, thereby verifying whether the observed outcomes in simulations are generalizable and not contingent on the specifics of the DGP.

\section{Experimentation : Calibration of a Random Forest}
\label{sec:expe-RF}

In this section, we focus on the application of RF models to predict the occurrence probability of a binary event $D \in \{0,1\}$. First, we compare RF when considered in the context of regression versus calibration. Then, we investigate the relationship between seeking accurate class predictions and the calibration of models.

\subsection{Classifier or Regressor}
\label{subsec:class-or-reg}

When RFs are employed to predict binary outcome variables, one can either train a classifier or a regressor. However, the scores returned by these two types of models are computed in very different ways. The score returned by an RF classifier corresponds to a majority class vote, whereas an RF regressor estimates probabilities by averaging initial class membership probabilities. 

In a classification forest comprising $M$ trees, each tree $m \in \{1,\ldots, M\}$ of the forest returns, for an observation $i$, a majority vote $\hat{d}_{m,i} \in \left\{0,1\right\}$. The predicted score for this observation is the average of the majority votes within the tree, $\hat{p}^{\text{class}}_i = \sum_{m=1}^{M} \hat{d}_{m,i}$. When training an RF regression on a binary target variable, the score returned by the forest for an observation $i$ is $\hat{p}^{\text{reg}}_i = \sum_{m=1}^{M} \hat{p}_{m,i}$, where $\hat{p}_{m,i}$ is calculated as the average of the observations in the terminal leaf to which observation $i$ belongs to the tree $m$.

Given the theoretical reliance of the forest regressor on probabilities, we anticipate that its calibration should be superior to that of the classifier, compared to \cite{calibRF2008}. 
Our objective is to compare the goodness-of-fit and calibration metrics of both models, both with and without the application of recalibration techniques on the estimated scores.

\paragraph{Data}

For our illustrations, we use data obtained from UCI \citep{misc_default_of_credit_card_clients_350}, presenting research customers' default payments in Taiwan. This dataset contains $n = 30,000$ instances and 23 numeric features. The outcome variable, corresponding to the observed default payment in next month, is positive in 22.12\% of cases. Following the methodology outlined in \cite{datapred2019}, we employ the Synthetic Minority Over-sampling Technique (SMOTE)~\citep{chawla2002smote} at a rate of 200\% to rebalance the data.

To predict default payments, we apply both RF algorithms on the rebalanced dataset.\footnote{We used the the \texttt{randomForest()} function from the R package \texttt{randomForest}.} The training set comprises 50\% of the dataset, while the remaining portion is further divided into a calibration set used to train a recalibrator and a test set to evaluate the goodness-of-fit and calibration of the models. Given the size of the dataset, this approach is feasible. For smaller datasets, \cite{caliForest2020} recommend using out-of-bag samples for training the recalibration method.

\paragraph{Methodology}

We conduct a grid search to find the set of hyperparameters that optimize a criterion: the out-of-bag MSE for the regressor, and the error rate for the classifier. The hyperparameters we vary include the number of trees, the number of variables considered for splitting, and the minimum number of observations in terminal nodes.\footnote{For further details on the grid search, refer to Section~\ref{sec-appendix-rf} in the appendices.} Once the hyperparameters are selected for both types of forests, the forests are used to make predictions on the remaining data not used in training. We then perform 200 simulations by splitting these predictions into two samples: a calibration sample and a test sample. 
For each simulation, we measure calibration and goodness-of-fit before and after recalibration, on both samples. This process yields a distribution of estimation quality and calibration metrics. 

\paragraph{Calibration}

Simulation results are presented in Figure~\ref{fig-recalib-rf-tuned-auc-lsc}. Considering the inefficiency of the ECE and Brier score, observed in the previous Section (see Figure~\ref{fig-calib-simuls-metrics}), our novel metric, LCS (Section~\ref{sec:measure-calibration}) is the only calibration metric shown in the Figure. For the goodness-of-fit metric, we only calculate the AUC, as it is invariant to the probability threshold.\footnote{Additional metrics are reported in the Appendix, in Figure~\ref{recalib-rf-tuned-metrics-gof} for goodness-of-fit, and in Figure~\ref{recalib-rf-tuned-metrics-calib} for calibration.} Also, the true MSE cannot be computed here, due to the unavailability of true data probability distributions.

\begin{figure}[ht]
\centering
\includegraphics[width = \linewidth]{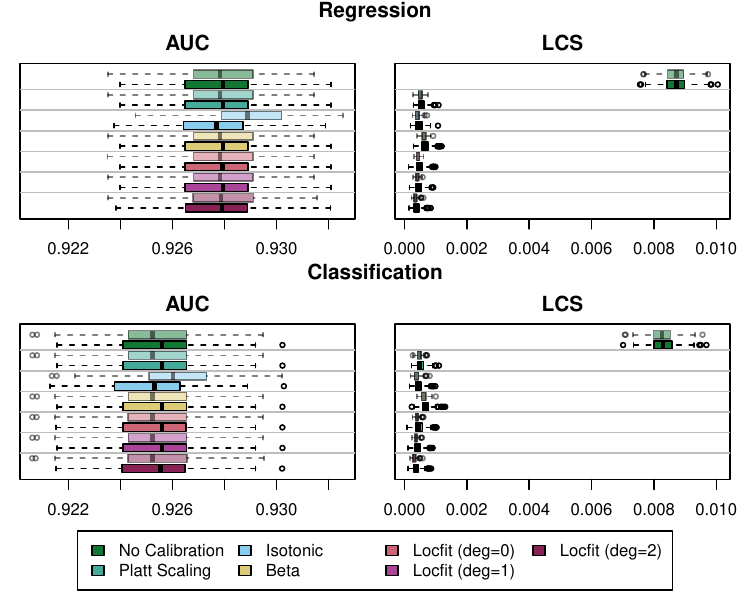}
\caption{Metrics Computed on 200 Replications, for the Regression Random Forest (top) and for the Classification Random Forest (bottom), on the Calibration (transparent colors) and on the Test Set (full colors).}\label{fig-recalib-rf-tuned-auc-lsc}
\end{figure}

Figure~\ref{fig-recalib-rf-tuned-auc-lsc} illustrates that prior to recalibration, the regressor exhibits slightly higher AUC and superior calibration than the classifier, thereby confirming the hypothesis outlined in Section~\ref{subsec:class-or-reg}. Moreover, although both algorithms appear to be well-calibrated, all recalibration methods enhance the calibration of these models without compromising the AUC. However, it is worth noting that isotonic regression seems to overfit in terms of AUC on the test set, an observation also documented by \cite{kull2017EJS}.

\subsection{Optimizing Class Predictions and Calibration}

ML models are typically fine-tuned to optimize hyperparameters to maximize goodness-of-fit in binary classification, often with less emphasis on calibration. However, the models that achieve the highest accuracy may not necessarily exhibit superior calibration. To investigate this issue, we assess the LCS across various goodness-of-fit levels, as measured by AUC, throughout the hyperparameter optimization process.

\begin{figure}[ht]
\centering
\includegraphics[width = \linewidth]{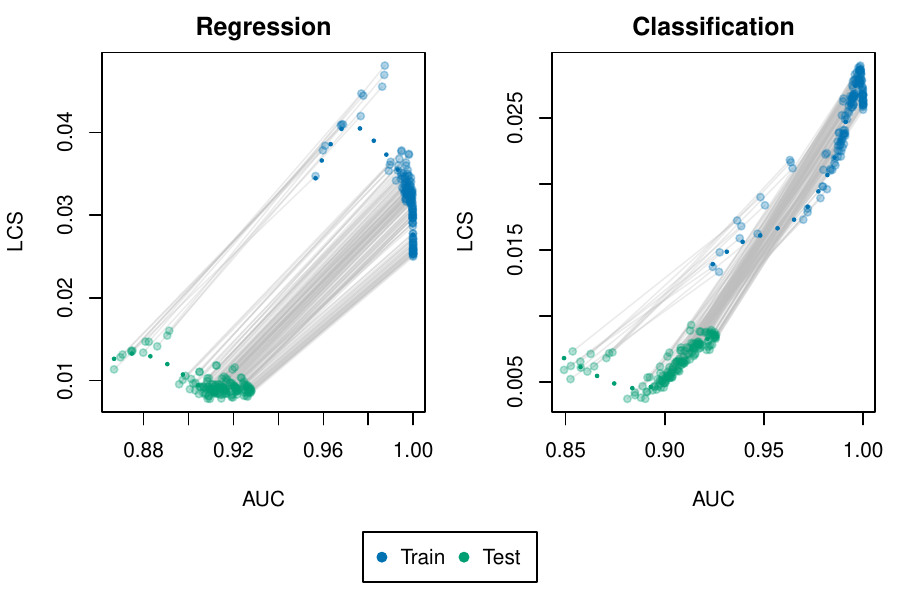}
\caption{Calibration vs. Performance of Random Forest Regression (left) and Random Forest Classification (right). Each point represents an estimation obtained from a set of hyperparameters. The gray lines help identify the estimations made within each sample using the same model.}\label{fig-rf-lcs-vs-auc}
\end{figure}

Figure~\ref{fig-rf-lcs-vs-auc} illustrates that for both types of RF, the order of calibration and performance metrics on the train sample is respected by the test sample. Furthermore, this Figure demonstrates that, optimizing the performance of an RF classifier with respect to the AUC reduces the calibration on the test sample. Thus, when using this type of model to directly employ its predicted scores, it may be necessary to consider both a performance metric and a calibration measure in the hyperparameter optimization process to consider these scores as the probabilities of belonging to the positive class. In contrast, in the case of regression, optimizing the AUC appears to enhance calibration on the test sample. 
However, opting for a model with an AUC not too close to 1 could be more advantageous for preventing overfitting. This involves sacrificing a small percentage of the AUC on the test sample while preserving good calibration, as depicted by the test points falling within an AUC range of 0.89 to 0.92.

\section{Conclusion}

This study aims to deepen our understanding of various calibration measures and methods for recalibrating binary classifiers. This is achieved by analyzing a simulated dataset generated from a logistic function, where the true probability distribution is known. We highlight the flexibility of synthetic data, unveiling nuances in calibration metrics, thereby identifying limitations in the Brier score and ECE within this context. We correct these limitations by introducing a novel calibration metric: the Local Calibration Score. This underscores the importance of local regression techniques for visualization and classifier recalibration. Experimental results using RF classifier and regressor to predict default risk demonstrate slightly better accuracy and calibration with the regressor. Notably, our evaluation of the LCS across various AUC levels during hyperparameter optimization reveals that RF classifiers achieving the highest goodness-of-fit do not necessarily exhibit superior calibration.

\bibliography{biblio.bib}

\newpage

\newgeometry{left=2.5cm,right=2.5cm, bottom=3cm, top=3cm}

\begin{center}
\rule{0.75\textwidth}{0.4pt}\\
\vspace{4mm}
{\scshape\Large From Uncertainty to Precision:\\ Enhancing Binary Classifier Performance through Calibration}\\
(Supplementary Material)\\
\rule{0.75\textwidth}{0.4pt}
\end{center}
\appendix

\section{Online Replication Book}

We provide replication material on GitHub\newline%
\href{https://github.com/fer-agathe/calibration_binary_classifier}{https://github.com/fer-agathe/calibration\_binary\_classifier}.

\section{Proofs}

\subsection{Proof of Proposition \ref{prop:logistic}}\label{proof:prop}

As a starting point, suppose that there is single feature, so that $\mathbf{x}$ can be denoted $x$. Suppose that $D|{X}={x} \sim \mathcal{B}\big(s({x})\big)$ where$$
s({x})=\frac{\exp[\beta_0+\beta_1{x}]}{1+\exp[\beta_0+\beta_1{x}]}
$$
with $\beta_1\neq 0$, and let $\widehat{\beta}_0$ and $\widehat{{\beta}}_1$, denote maximum likelihood estimators, so that the model is well specified. Then, for any ${x}$, the score is
$$
\hat{s}({x})=\frac{\exp[\hat\beta_0+\hat\beta_1{x}]}{1+\exp[\hat\beta_0+\hat\beta_1{x}]}.
$$
Here both $s$ and $\hat{s}$ are continuous and invertible. Given $p\in(0,1)$, $\{\hat{s}({x})=p\}$ corresponds to $
\{x=\hat{s}^{-1}(p)\},$
since mapping $\hat{s}$ is one-to-one. Thus, $D$ conditional on $\{\hat{s}({x})=p\}$ is therefore a Bernoulli variable with mean $s(\hat{s}^{-1}(p))$.
And because $\hat{s}$ and $s$ are continuous and bijective functions
$$
\begin{cases}
    \hat{\beta}_0\to\beta_0\\
    \hat{\beta}_1\to\beta_1
\end{cases}
\text{ as }n\to\infty
~\Longrightarrow~\forall x,p~
\begin{cases}
    \hat{s}(x)\to s(x)\\
    \hat{s}^{-1}(p)\to s^{-1}(p)
\end{cases}
\text{ as }n\to\infty
$$
since the model is well specified, where the convergence is in probability here. 
Thus
$$
s(\hat{s}^{-1}(p))\to p\text{ as }n\to\infty
$$
and
$$
\mathbb{E}[D|\hat{s}(x)=y]=p(\hat{s}^{-1}(p))\to p\text{ as }n\to\infty.
$$
This property holds in higher (fixed) dimensions, if the model is well specified, since functions are continuous and invertible (linear models with continuous and invertible link functions). The case where the dimension of $\mathbf{x}$ increases with $n$ is discussed in \cite{bai2021don}.

\section{Calibration with Simulated Data}

This appendix provides additional figures regarding the 200 simulations made using the data generating process described in the main text.

We generated data from the data generating process described in Equations~\ref{eq-true-propensity} and~\ref{eq-propensity-eta} and applied transformations either on the probabilities (varying $\alpha$ in Equation~\ref{eq-true-propensity-decalibrated}) or on the linear predictor (varying $\gamma$ in Equation~\ref{eq-propensity-eta-decalibrated}).

For each value of $\alpha=\{1/3, 1, 3\}$ and $\gamma=\{1/3, 1, 3\}$, we generated 200 datasets. The distribution of a single dataset for each value of $\alpha$ and $\gamma$ is shown in Figure~\ref{fig-distrib-probas-transfo}.

\begin{figure}[ht]
\centering
\includegraphics[width = .7\linewidth]{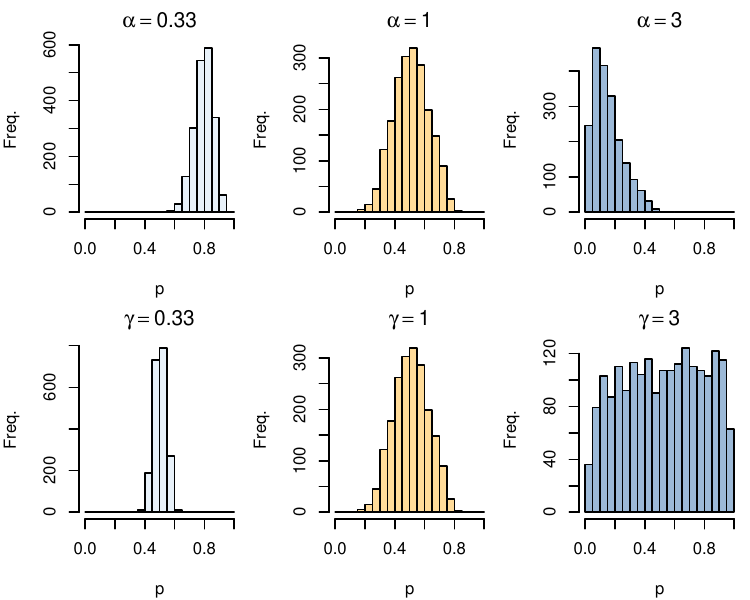}
\caption{Distribution of Distorted Probabilities $p^u$ Depending on $\alpha$ (top) or $\gamma$ (bottom).}\label{fig-distrib-probas-transfo}
\end{figure}


\def\figwidth{.45\textwidth}

\subsection{Metrics}

Figure~\ref{fig-recalib-simuls-calib-metrics} displays different calibration metrics (True MSE in Figure~\ref{fig-recalib-simuls-metrics-truemse}, the Brier Score in Figure~\ref{fig-recalib-simuls-metrics-brier}, the Expected Calibration Error in Figure~\ref{fig-recalib-simuls-metrics-ece}, and the Local Calibration Score in Figure~\ref{fig-recalib-simuls-metrics-lcs}), for the complete set of simulations. Each graph of each panel in the Figure shows the distribution of a metric computed over 200 replications of the simulations for a scenario in which the scores are distorted, on the calibration set (transparent colors) and on the test set (solid colors). In each case, the metric is computed using either the \textcolor{bleuTOL}{true probabilities $p$} which represents the ground truth, the \textcolor{vertTOL}{uncalibrated scores} $p^u$, or the recalibrated scores, $p^c$. The recalibration techniques are the following: \textcolor{vertClairTOL}{Platt scaling}, \textcolor{bleuClairTOL}{Isotonic regression}, \textcolor{sableTOL}{Beta calibration}, and Local regression with varying degrees: \textcolor{parmeTOL}{degree 0}, \textcolor{magentaTOL}{degree 1}, and \textcolor{roseTOL}{degree 2}. See Table~\ref{tab:recalibration-techniques} for a summary. This figure complements Figure~\ref{fig-recalib-simuls-metrics-gamma3} from the main text. However, unlike in Figure~\ref{fig-recalib-simuls-metrics-gamma3}, the values are not expresses here as the difference with the value observed in the case where the scores used are the \textcolor{vertTOL}{uncalibrated scores} $p^u$.

\begin{table}[ht]
    \centering\small
    \caption{Recalibration Techniques Used in the Simulations}
    \begin{tabular}{lll}
        \hline\hline
        \textbf{Legend} & \textbf{Scores Used} & \textbf{Recalibration Technique} \\
        \midrule
        \fcolorbox{black}{bleuTOL}{\rule{0pt}{5pt}\rule{5pt}{0pt}}\quad \textcolor{bleuTOL}{\texttt{True Prob.}} & Simulated true prob. & No recalibration technique \\
        \fcolorbox{black}{vertTOL}{\rule{0pt}{5pt}\rule{5pt}{0pt}}\quad\textcolor{vertTOL}{\texttt{No Calibration}} & Transformed prob. & No recalibration technique \\
        \fcolorbox{black}{vertClairTOL}{\rule{0pt}{5pt}\rule{5pt}{0pt}}\quad\textcolor{vertClairTOL}{\texttt{Platt Scaling}} & Transformed prob. & Platt Scaling \\
        \fcolorbox{black}{bleuClairTOL}{\rule{0pt}{5pt}\rule{5pt}{0pt}}\quad\textcolor{bleuClairTOL}{\texttt{Isotonic}} & Transformed prob. & Isotonic regression \\
        \fcolorbox{black}{sableTOL}{\rule{0pt}{5pt}\rule{5pt}{0pt}}\quad\textcolor{sableTOL}{\texttt{Beta}} & Transformed prob. & Beta calibration \\
        \fcolorbox{black}{parmeTOL}{\rule{0pt}{5pt}\rule{5pt}{0pt}}\quad\textcolor{parmeTOL}{\texttt{Locfit (deg = 0)}} & Transformed prob. & Local reg. with degree 0 \\
        \fcolorbox{black}{magentaTOL}{\rule{0pt}{5pt}\rule{5pt}{0pt}}\quad\textcolor{magentaTOL}{\texttt{Locfit (deg = 1)}} & Transformed prob. & Local reg. with degree 1 \\
        \fcolorbox{black}{roseTOL}{\rule{0pt}{5pt}\rule{5pt}{0pt}}\quad\textcolor{roseTOL}{\texttt{Locfit (deg = 2)}} & Transformed prob. & Local reg. with degree 2 \\
        \hline\hline
    \end{tabular}
    \label{tab:recalibration-techniques}
\end{table}

\begin{figure}[ht]%
\centering
\subfloat[True MSE.]{
    \includegraphics[width = \figwidth]{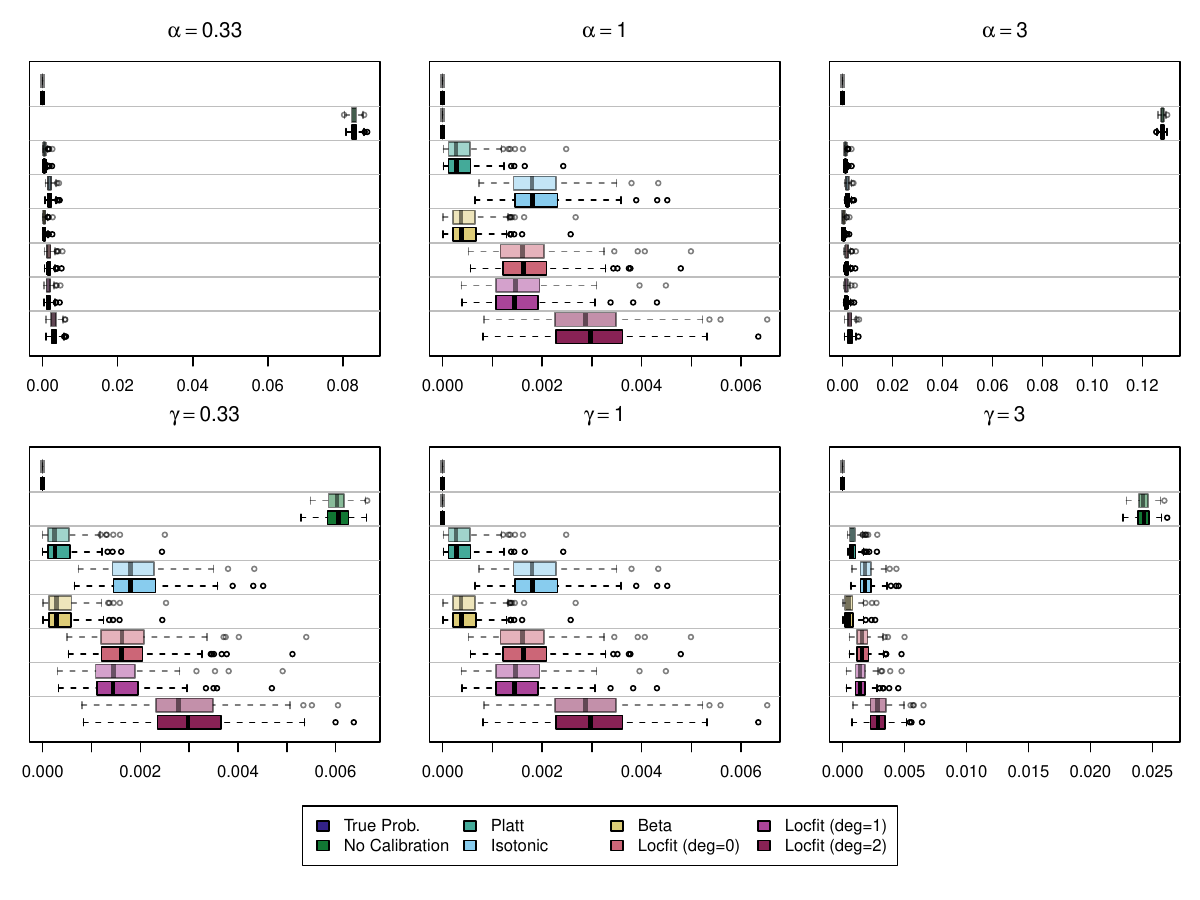}\label{fig-recalib-simuls-metrics-truemse}
}\qquad
\subfloat[Brier Score.]{
    \includegraphics[width = \figwidth]{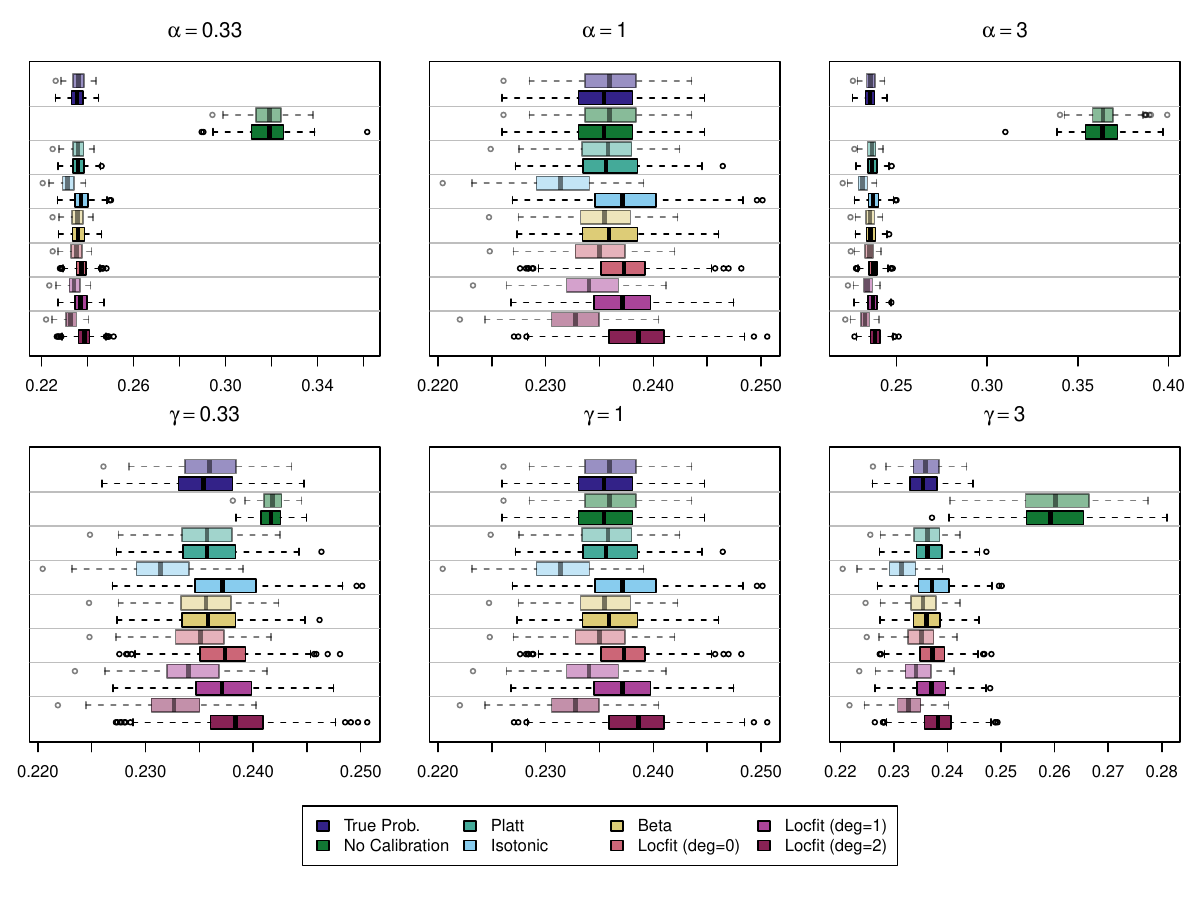}\label{fig-recalib-simuls-metrics-brier}
}\\
\subfloat[Expected Calibration Error.]{
    \includegraphics[width = \figwidth]{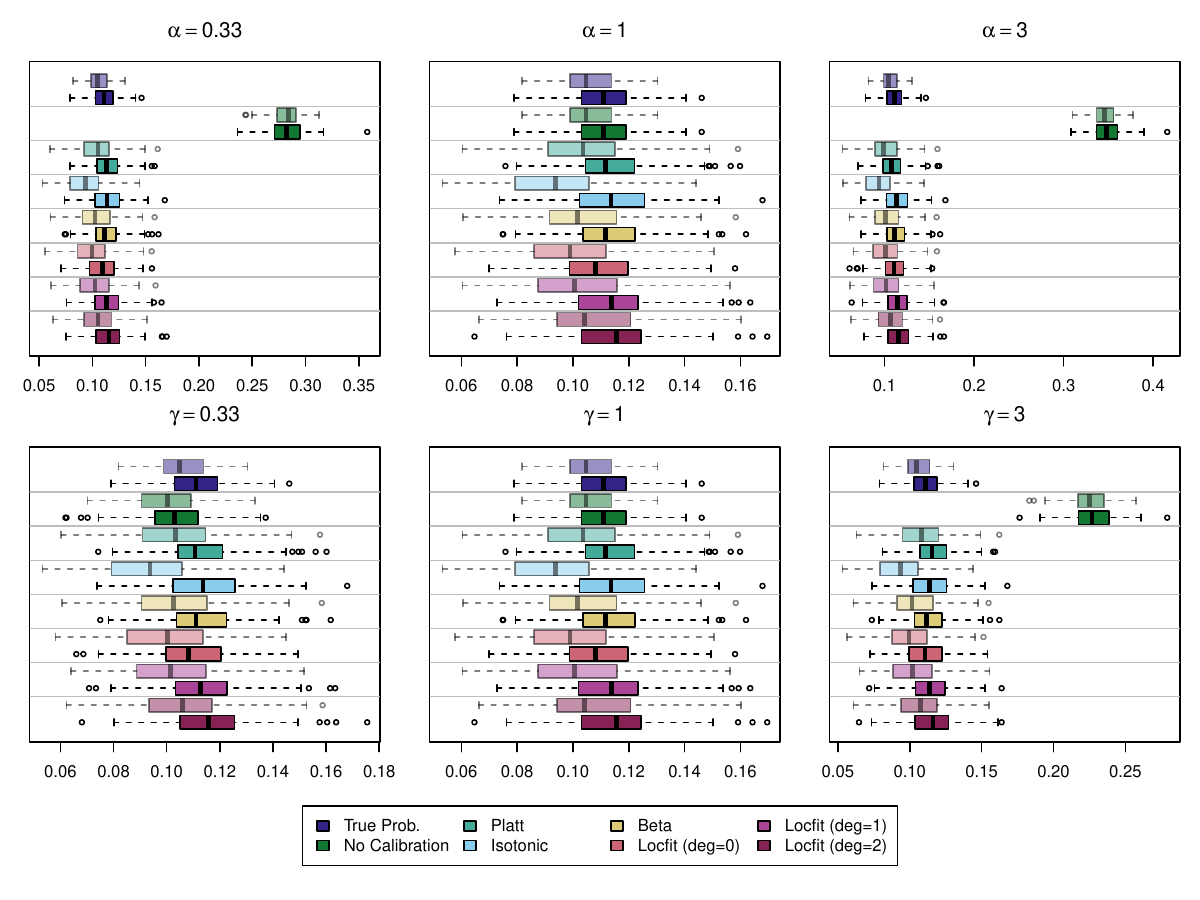}\label{fig-recalib-simuls-metrics-ece}
}\qquad
\subfloat[Local Calibration Score.]{
    \includegraphics[width = \figwidth]{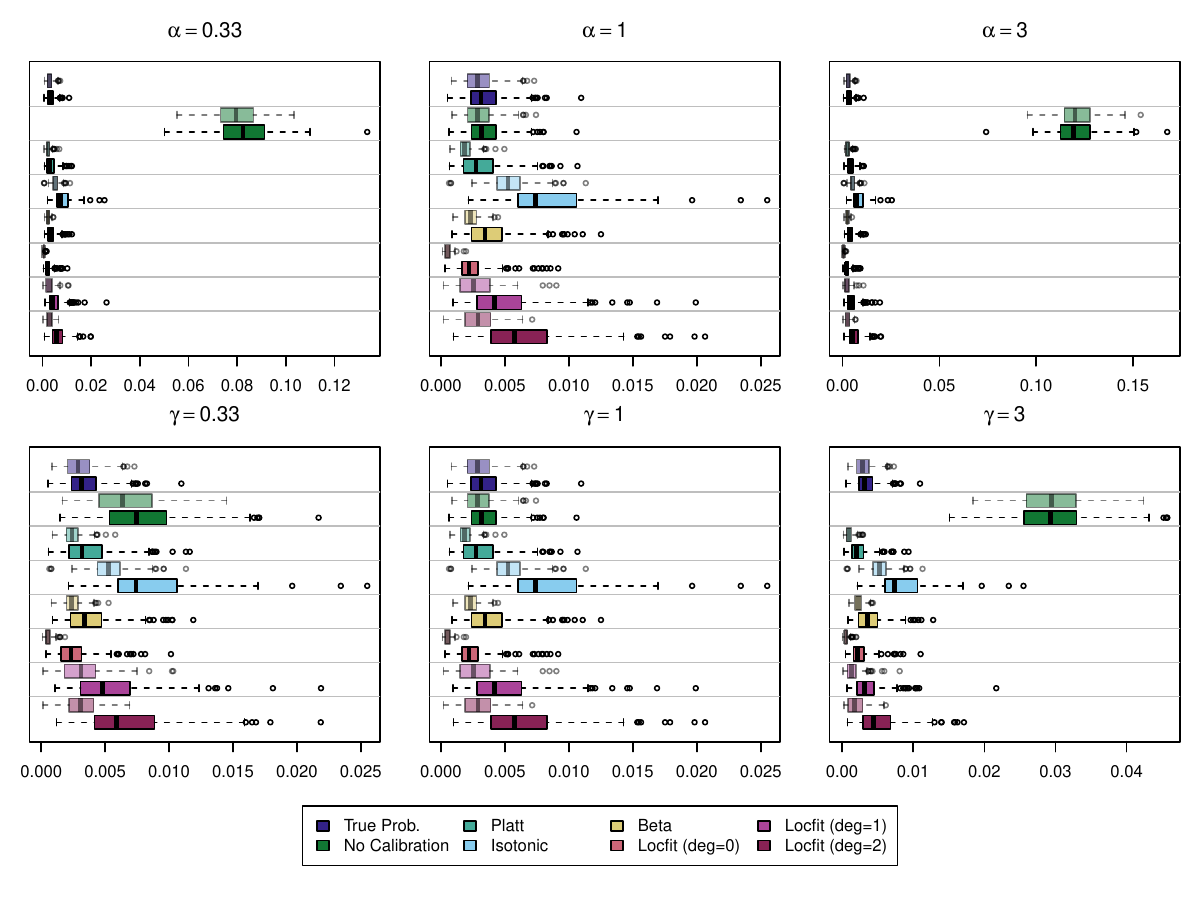}\label{fig-recalib-simuls-metrics-lcs}
}\\
\caption{Calibration Metrics on 200 Simulations for each Value of $\alpha$ (top) or $\gamma$ (bottom), on the Calibration (transparent colors) and on the Test Set (full colors). The metrics are computed for different definitions of the scores: using the \textcolor{bleuTOL}{true probabilities}, the \textcolor{vertTOL}{non calibrated scores}, or the recalibrated scores.}
\label{fig-recalib-simuls-calib-metrics}
\end{figure}

\subsection{Calibration Curves with Quantile Binning}

Figure~\ref{fig-calib-simuls-calib-curve-quant} illustrates, mirroring Figure~\ref{fig-calib-simuls-calib-curve-locfit} the overlay of 200 calibration curves computed using bins defined by quantiles, for each type of score distortion (varying $\alpha$ or $\gamma$). The distribution of true probabilities is depicted at the top of each graph of the figure.

\begin{figure}[ht]
\centering
\includegraphics[width = .6\linewidth]{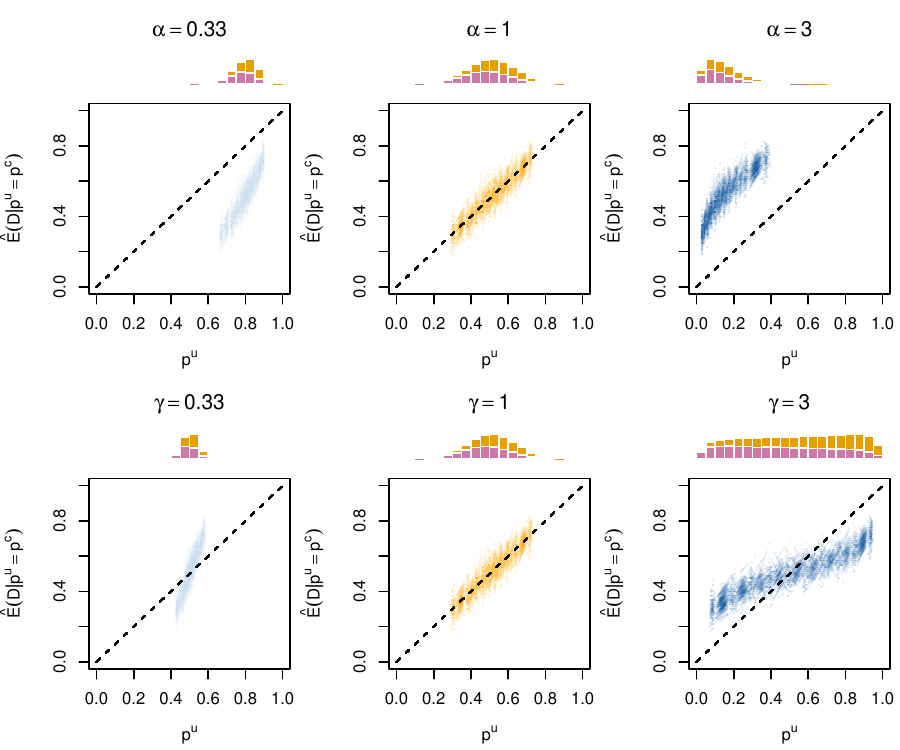}
\caption{Calibration Curves Defined with Bins, on 200 Simulations for each Value of $\alpha$ (top) or $\gamma$ (bottom).}\label{fig-calib-simuls-calib-curve-quant}
\end{figure}

The calibration curves computed after recalibrating the scores are calculated on the \textcolor{wongOrange}{calibration set (shown in orange)}, and on the \textcolor{wongGreen}{test set (shown in green)}, for each poor calibration scenario, where $\alpha$ or $\gamma$ vary. When $\alpha=1$ or $\gamma=1$, the transformed scores $p^u$ are in fact equal to the true probabilities $p$. We consider different scenarios. In Figure~\ref{fig-recalib-simuls-calib-curve-quant-true_prob}, the scores are the true probabilities $p$ and no recalibration technique is employed. In Figure~\ref{fig-recalib-simuls-calib-curve-quant-no_calib}, the scores are the uncalibrated values $p^u$ and no recalibration technique is employed either. Then, we consider recalibration techniques on these uncalibrated values: Platt scaling (Figure~\ref{fig-recalib-simuls-calib-curve-quant-platt}), isotonic regression (Figure~\ref{fig-recalib-simuls-calib-curve-quant-isotonic}), beta calibration (Figure~\ref{fig-recalib-simuls-calib-curve-quant-beta}), and local regression (Figure~\ref{fig-recalib-simuls-calib-curve-quant-locfit_0} for degree 0, Figure~\ref{fig-recalib-simuls-calib-curve-quant-locfit_1} for degree 1, and Figure~\ref{fig-recalib-simuls-calib-curve-quant-locfit_2} for degree 2).

\def\figwidth{.45\textwidth}

\begin{figure}[ht]%
\centering
\subfloat[\textcolor{bleuTOL}{True Probabilities}.]{
    \includegraphics[width = \figwidth]{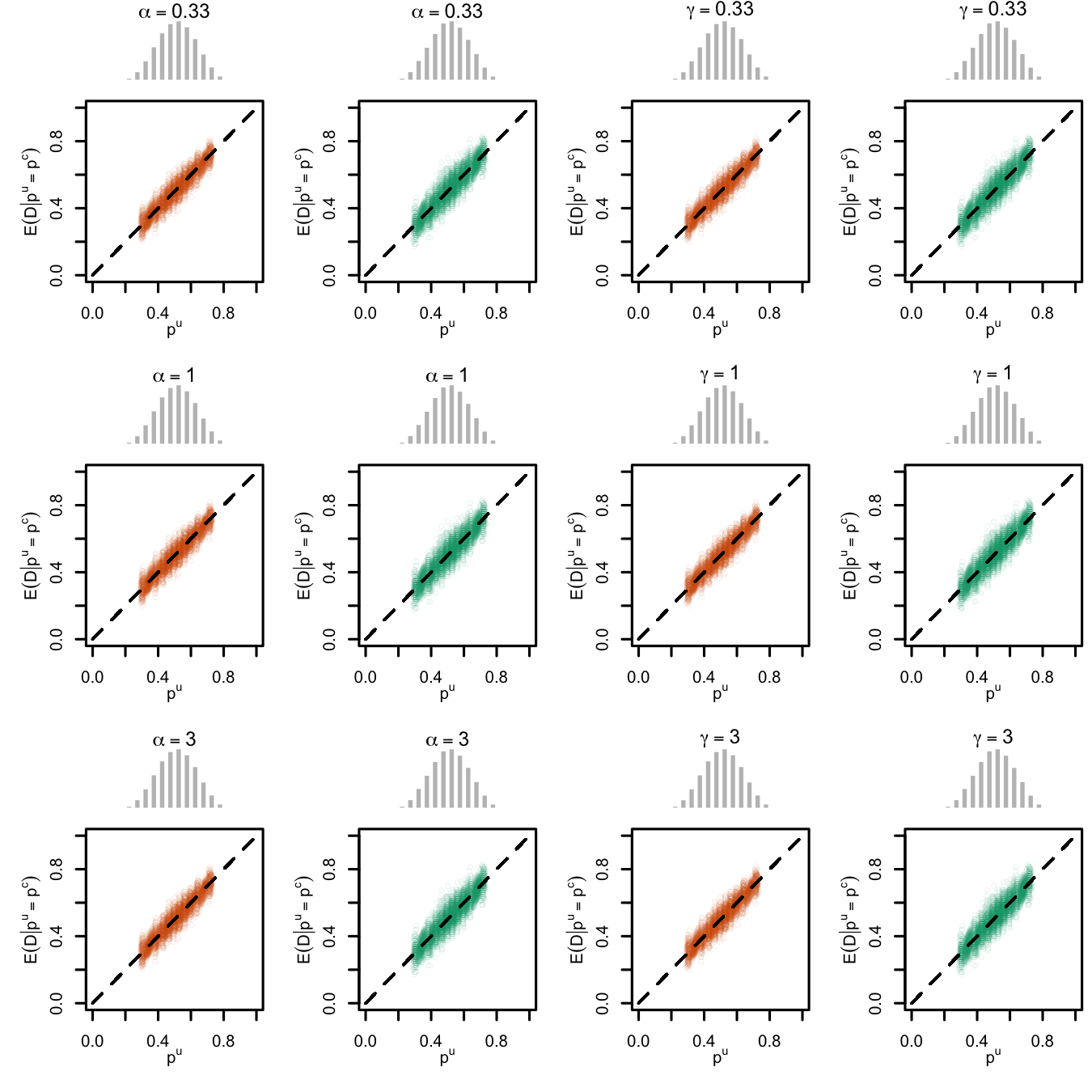}\label{fig-recalib-simuls-calib-curve-quant-true_prob}
}\qquad
\subfloat[\textcolor{vertTOL}{Uncalibrated Scores}.]{
    \includegraphics[width = \figwidth]{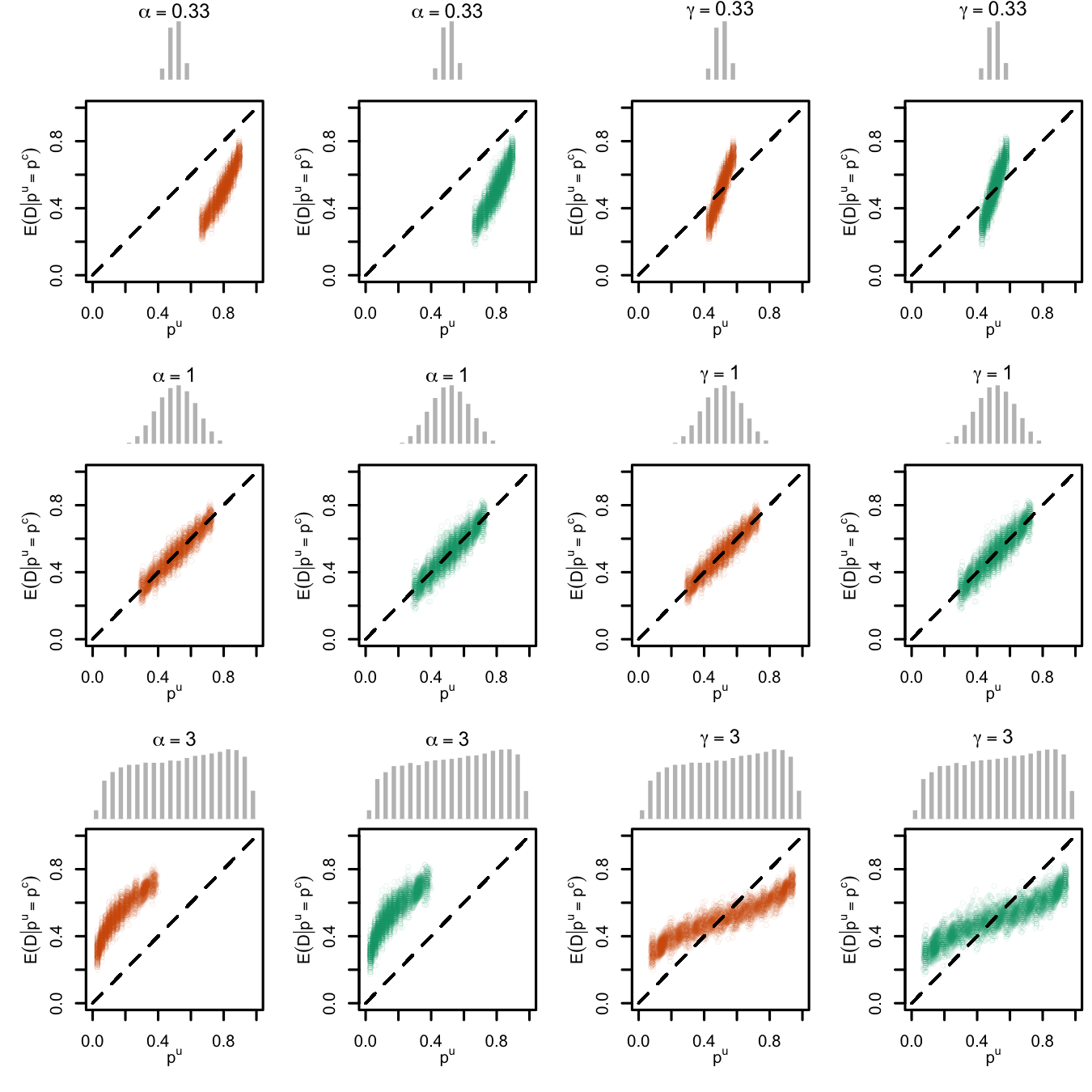}\label{fig-recalib-simuls-calib-curve-quant-no_calib}
}\\
\caption{Calibration Curves Calculated with \textcolor{bleuTOL}{\textbf{True Probabilities} (Panel a)} or with \textcolor{vertTOL}{\textbf{Uncalibrated Scores} (Panel B)} as the Scores. The curves are obtained with \textbf{quantile binning}, for the \textcolor{wongOrange}{calibration set (orange)} and for the \textcolor{wongGreen}{test set (green)} for varying values of $\alpha$ and $\gamma$. The curves of the 200 replications of the simulations are superimposed. The histogram on top of each graph show the distribution of the \textcolor{colUncalibrated}{true probabilities (Panel a) and of the uncalibrated scores (Panel b) (in gray)}.}
\label{fig-recalib-simuls-calib-curve-quant-true_prob-no-calib}
\end{figure}

\begin{figure}[ht]%
\centering
\subfloat[\textcolor{vertClairTOL}{Platt Scaling}.]{
    \includegraphics[width = \figwidth]{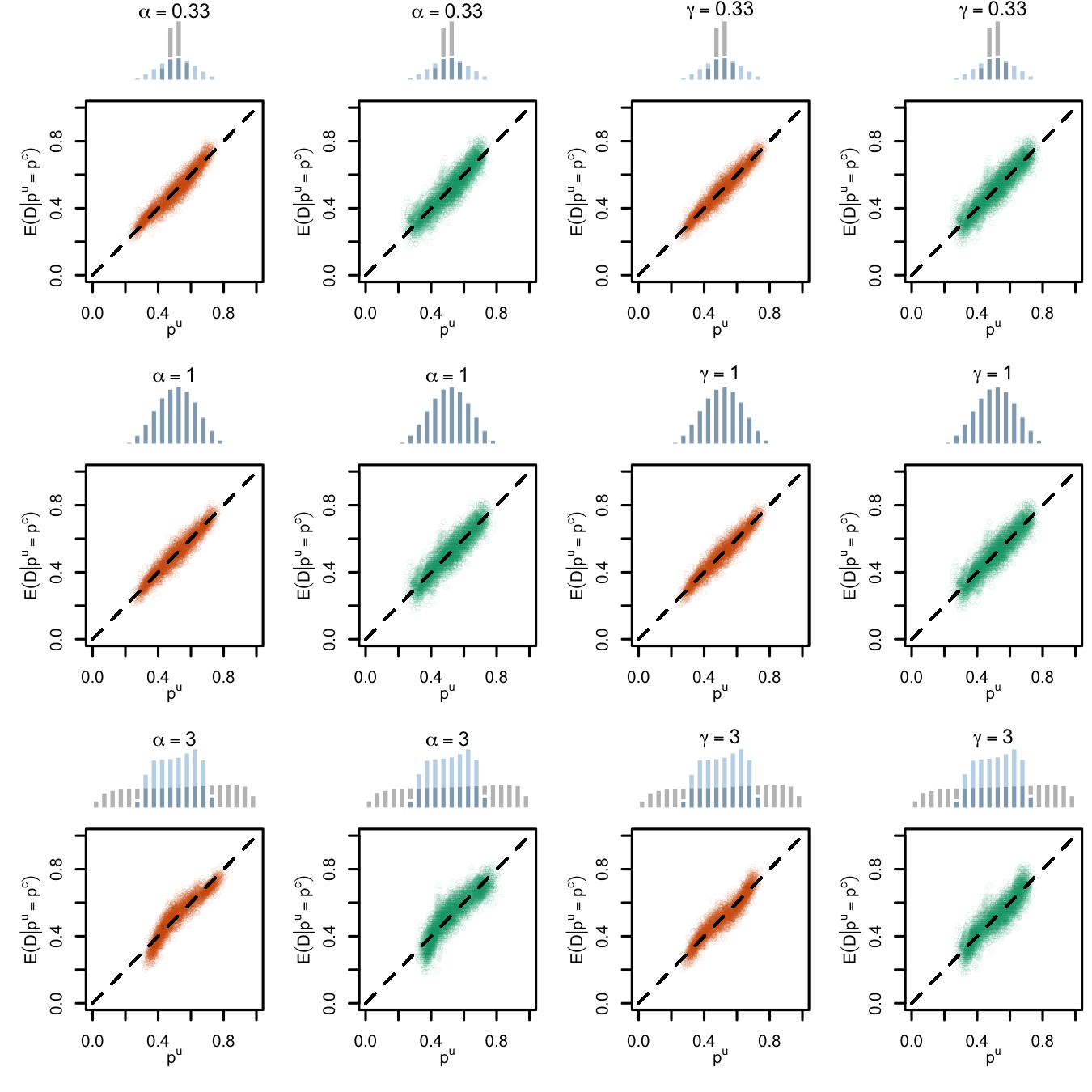}\label{fig-recalib-simuls-calib-curve-quant-platt}
}\qquad
\subfloat[\textcolor{bleuClairTOL}{Isotonic Regression}.]{
    \includegraphics[width = \figwidth]{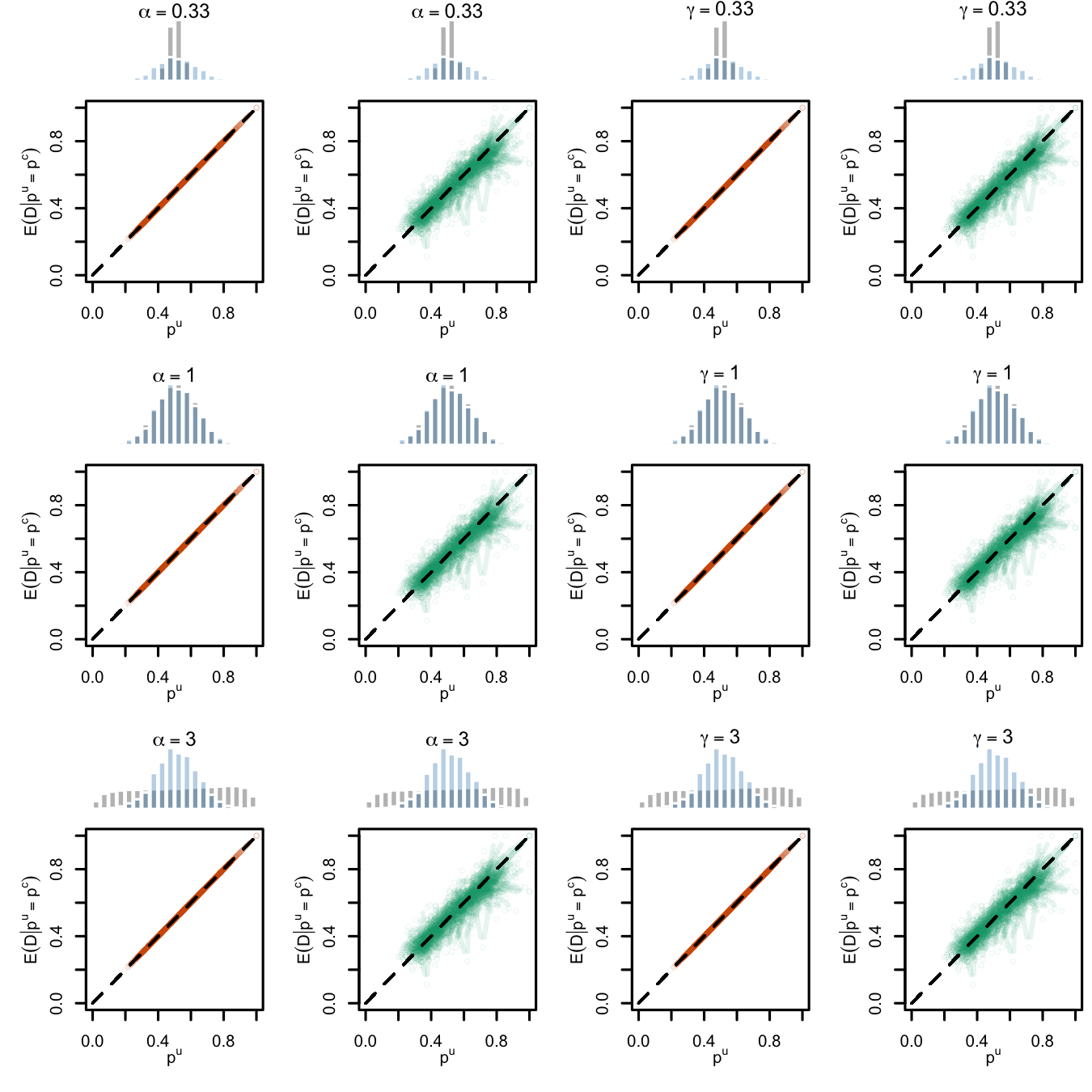}\label{fig-recalib-simuls-calib-curve-quant-isotonic}
}\\
\subfloat[\textcolor{sableTOL}{Beta Calibration}.]{
    \includegraphics[width = \figwidth]{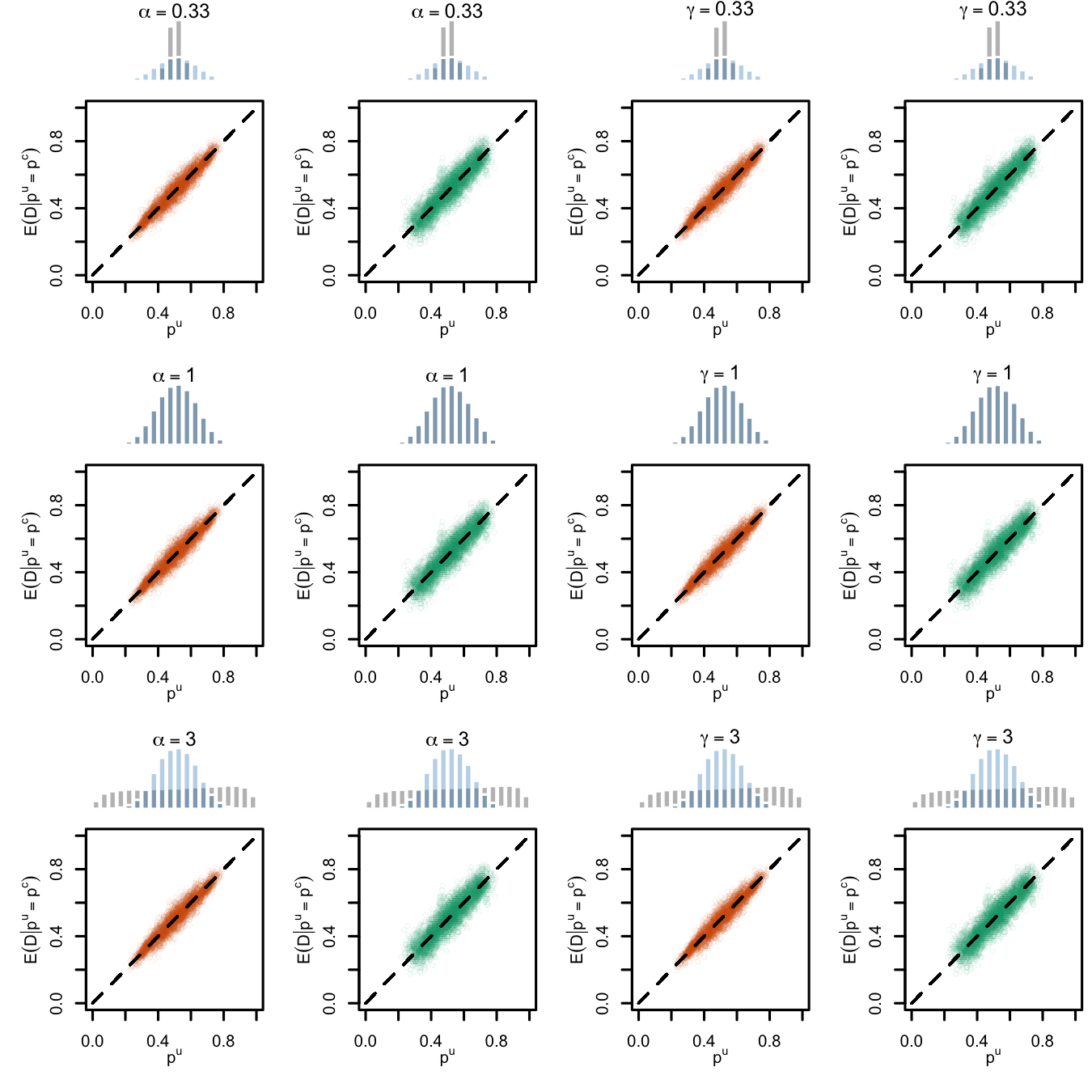}\label{fig-recalib-simuls-calib-curve-quant-beta}
}
\caption{Calibration Curves Calculated with \textbf{Scores Recalibrated} Using \textcolor{vertClairTOL}{\textbf{Platt Scaling (Panel a)}}, \textcolor{bleuClairTOL}{\textbf{Isotonic Regression (Panel b)}}, or \textcolor{sableTOL}{\textbf{Beta Calibration (Panel c)}}. The curves are obtained with \textbf{quantile binning}, for the \textcolor{wongOrange}{calibration set (orange)} and for the \textcolor{wongGreen}{test set (green)} for varying values of $\alpha$ and $\gamma$. The curves of the 200 replications of the simulations are superimposed. The histogram on top of each graph show the distribution of the \textcolor{colUncalibrated}{uncalibrated scores (gray)}, and that of the \textcolor{colRecalibrated}{calibrated scores (blue)}.}
\label{fig-recalib-simuls-calib-curve-quant-platt-isotonic}
\end{figure}

\begin{figure}[ht]%
\centering
\subfloat[\textcolor{parmeTOL}{Degree 0}.]{
    \includegraphics[width = \figwidth]{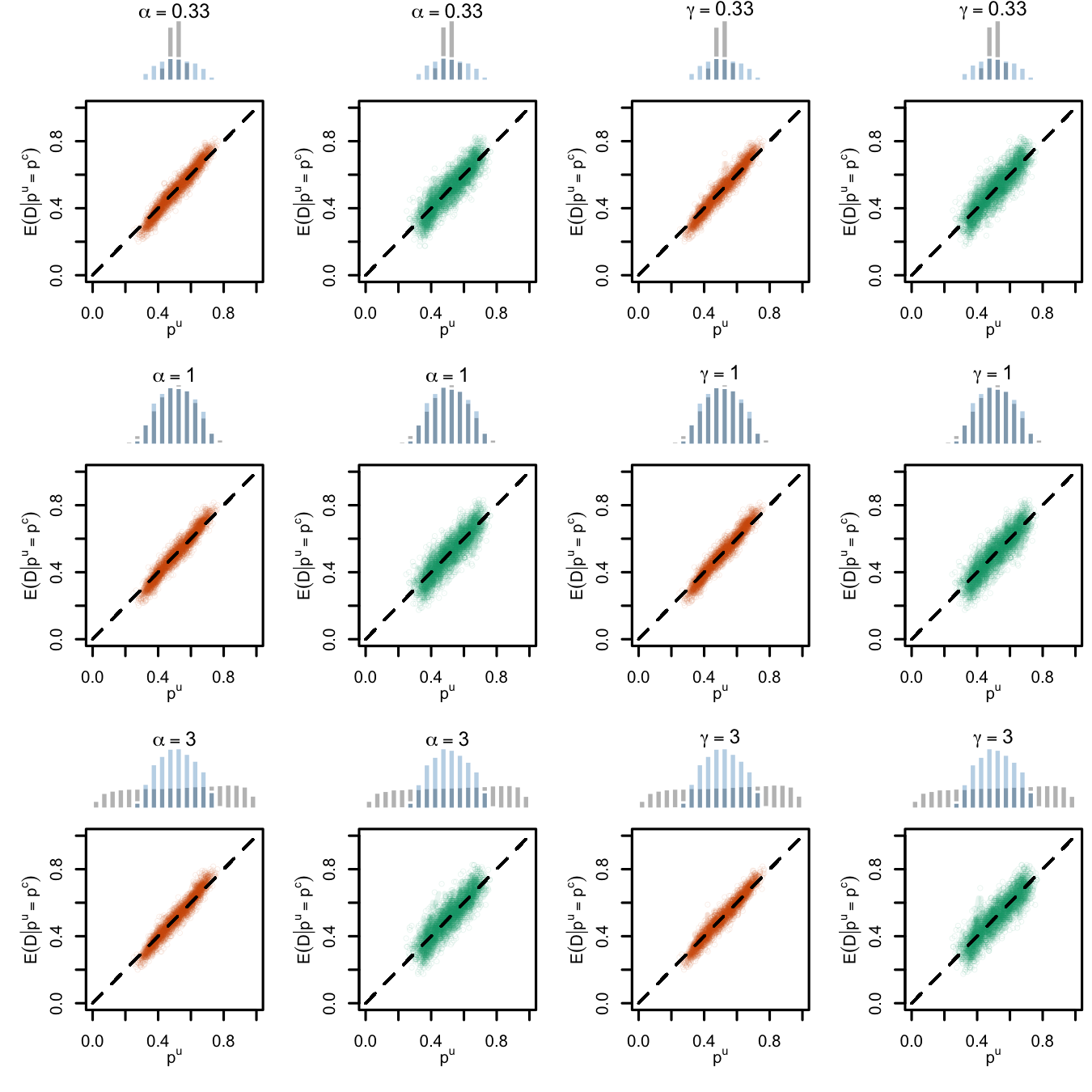}\label{fig-recalib-simuls-calib-curve-quant-locfit_0}
}\qquad
\subfloat[\textcolor{magentaTOL}{Degree 1}.]{
    \includegraphics[width = \figwidth]{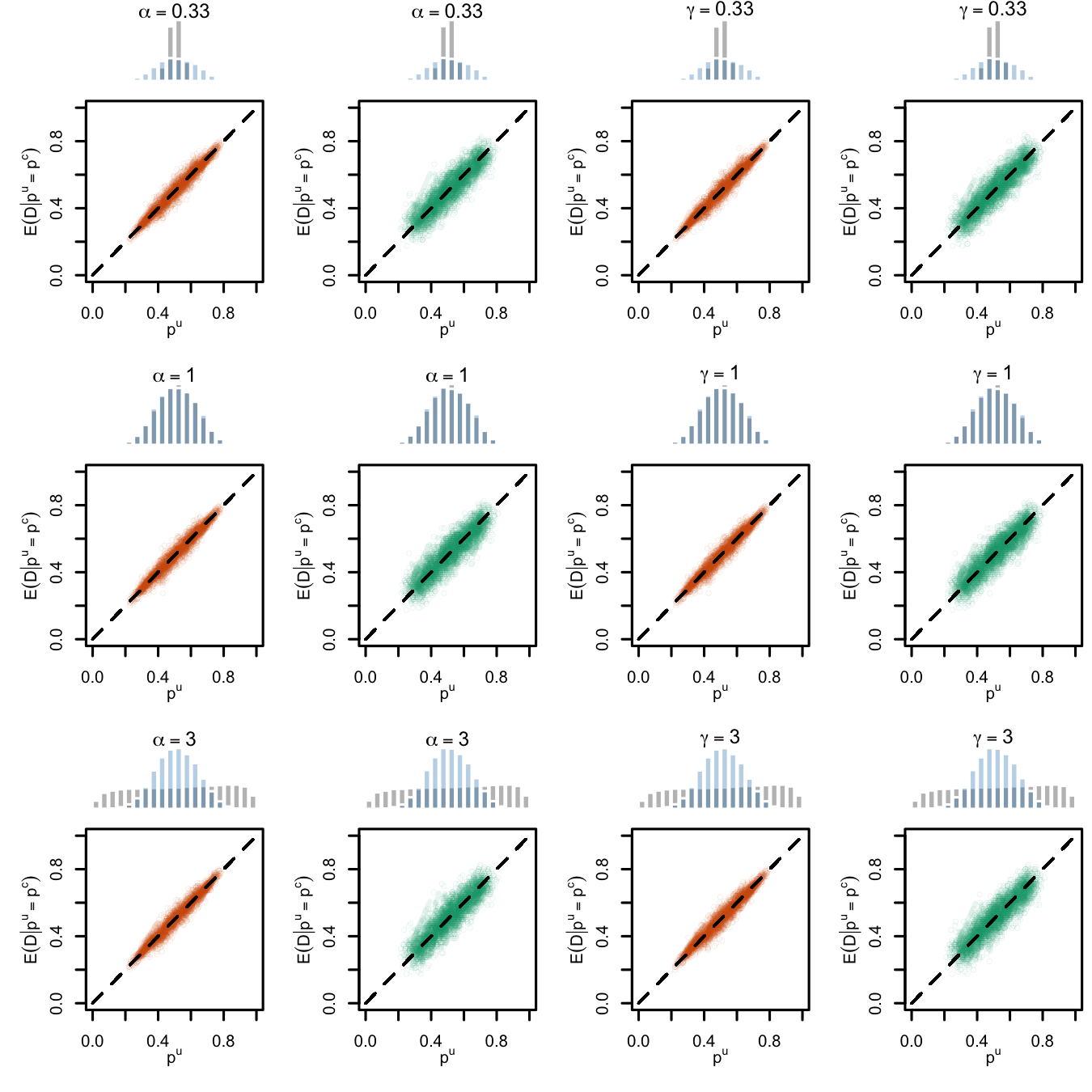}\label{fig-recalib-simuls-calib-curve-quant-locfit_1}
}\\
\subfloat[\textcolor{roseTOL}{Degree 2}.]{
    \includegraphics[width = \figwidth]{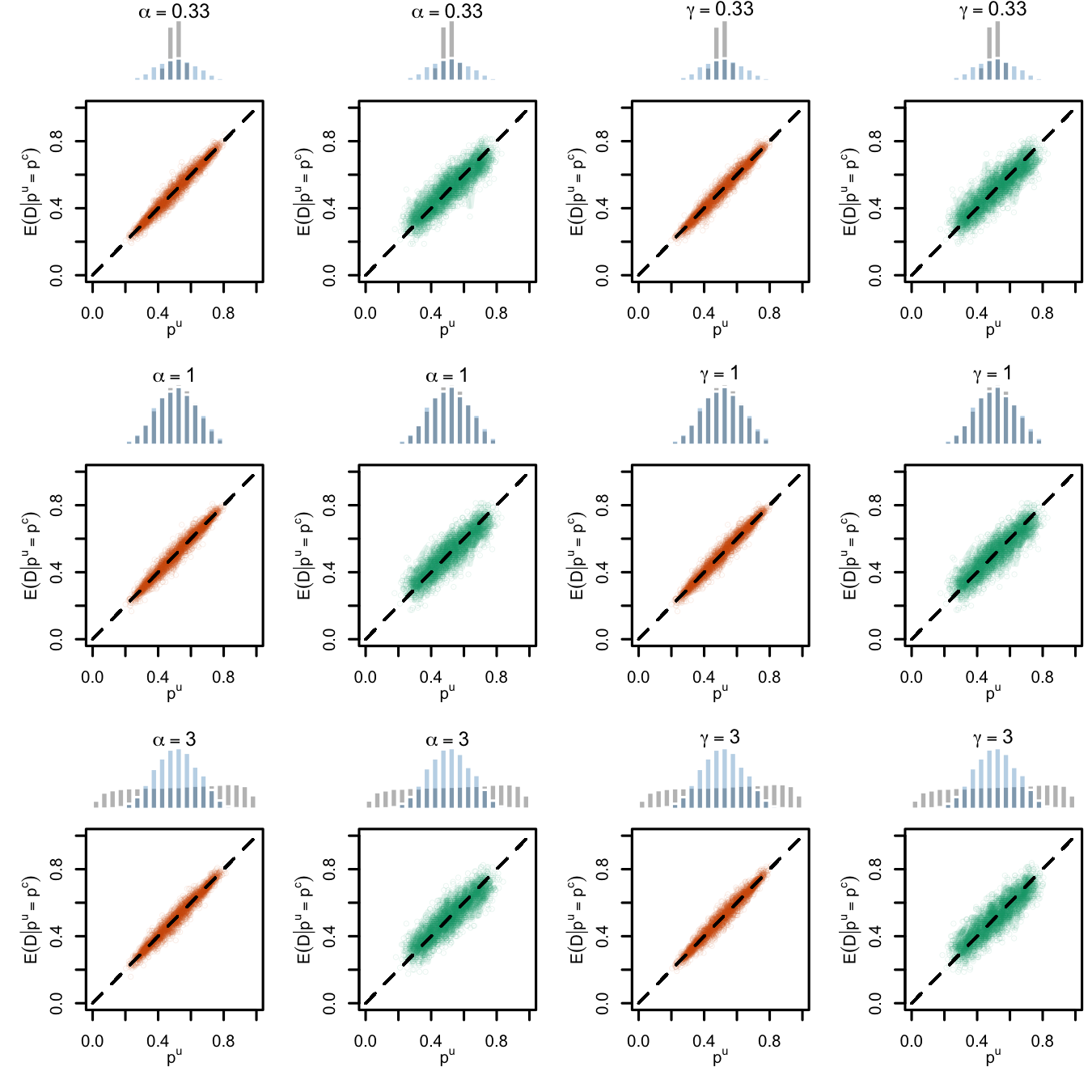}\label{fig-recalib-simuls-calib-curve-quant-locfit_2}
}
\caption{Calibration Curves Calculated with \textbf{Scores Recalibrated} Using \textbf{Local Regression} with \textcolor{parmeTOL}{Degree 0 (Panel a)}, \textcolor{magentaTOL}{Degree 1 (Panel b)}, or with \textcolor{roseTOL}{Degree 2 (Panel c)}. The curves are obtained with \textbf{quantile binning}, for the \textcolor{wongOrange}{calibration set (orange)} and for the \textcolor{wongGreen}{test set (green)} for varying values of $\alpha$ and $\gamma$. The curves of the 200 replications of the simulations are superimposed.The histogram on top of each graph show the distribution of the \textcolor{colUncalibrated}{uncalibrated scores}, and that of the \textcolor{colRecalibrated}{calibrated scores}.}
\label{fig-recalib-simuls-calib-curve-quant-locfit}
\end{figure}

\subsection{Calibration Curves with Local Regression}

In addition to the calibration curves computed using the quantile binning approach, we provide the calibration curves computed on the simulations using local regression on the \textcolor{wongOrange}{calibration set (shown in orange)} and on the \textcolor{wongGreen}{test set (shown in green)}, depending on the scores used to obtain the calibration curves: true probabilities (Figure~\ref{fig-recalib-simuls-calib-curve-locfit-true_prob}), uncalibrated transformed probabilities (Figure~\ref{fig-recalib-simuls-calib-curve-locfit-no_calib}), transformed probabilities recalibrated with Platt Scaling (Figure~\ref{fig-recalib-simuls-calib-curve-locfit-platt}), transformed probabilities recalibrated with isotonic regression (Figure~\ref{fig-recalib-simuls-calib-curve-locfit-isotonic}), transformed probabilities recalibrated with beta calibration (Figure~\ref{fig-recalib-simuls-calib-curve-locfit-beta}, and transformed probabilities recalibrated with local regression with degree 0 (Figure~\ref{fig-recalib-simuls-calib-curve-locfit-locfit-0}), with and with degree 2 (Figure~\ref{fig-recalib-simuls-calib-curve-locfit-locfit-2}).

\begin{figure}[ht]%
\centering
\subfloat[\textcolor{bleuTOL}{True Probabilities}.]{
    \includegraphics[width = \figwidth]{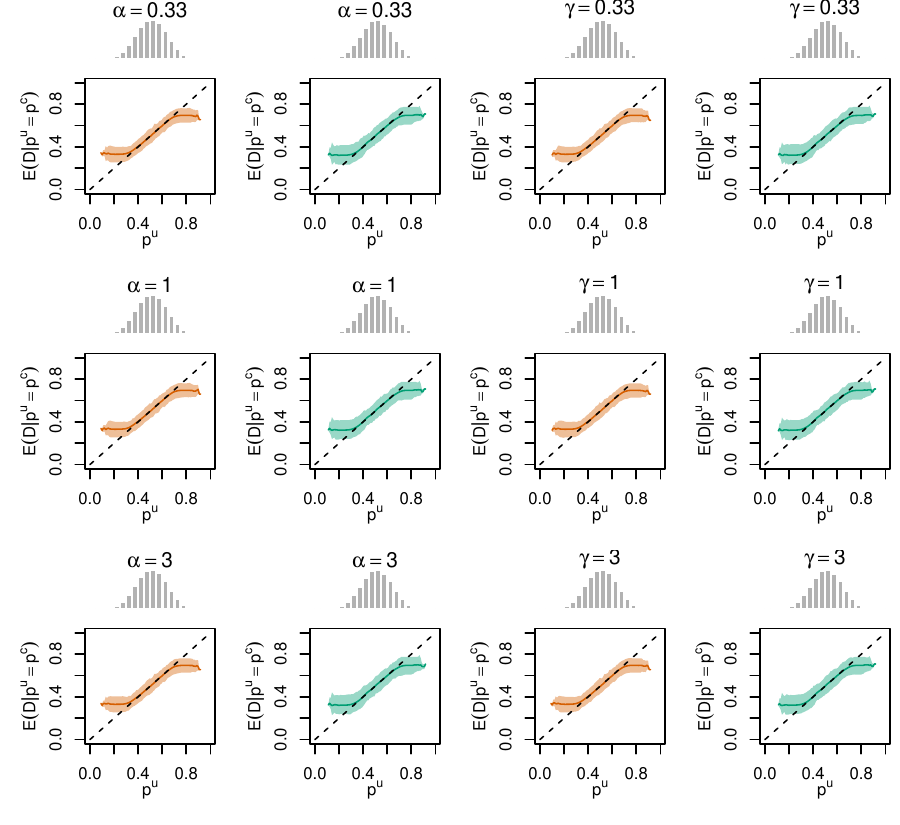}\label{fig-recalib-simuls-calib-curve-locfit-true_prob}
}\qquad
\subfloat[\textcolor{vertTOL}{Uncalibrated Scores}.]{
    \includegraphics[width = \figwidth]{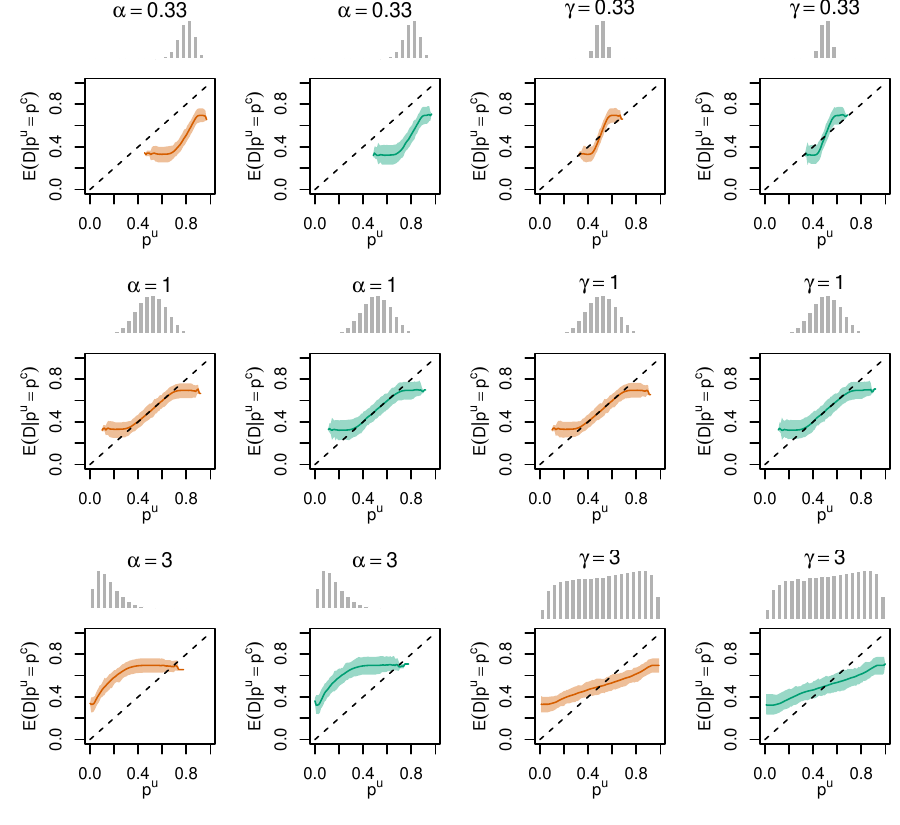}\label{fig-recalib-simuls-calib-curve-locfit-no_calib}
}\\
\caption{Calibration Curves Calculated with \textcolor{bleuTOL}{\textbf{True Probabilities} (Panel a)} or with \textcolor{vertTOL}{\textbf{Uncalibrated Scores} (Panel B)} as the Scores. The curves are obtained with \textbf{a local regression}, for the \textcolor{wongOrange}{calibration set (orange)} and for the \textcolor{wongGreen}{test set (green)} for varying values of $\alpha$ and $\gamma$. The curves are the average values obtained on 200 replications of the simulations, the bands correspond to 95\% bootstrap interval. The histogram on top of each graph show the distribution of the \textcolor{colUncalibrated}{true probabilities}.}
\label{fig-recalib-simuls-calib-curve-locfit-true_prob-no-calib}
\end{figure}

\begin{figure}[ht]%
\centering
\subfloat[\textcolor{vertClairTOL}{Platt Scaling}.]{
    \includegraphics[width = \figwidth]{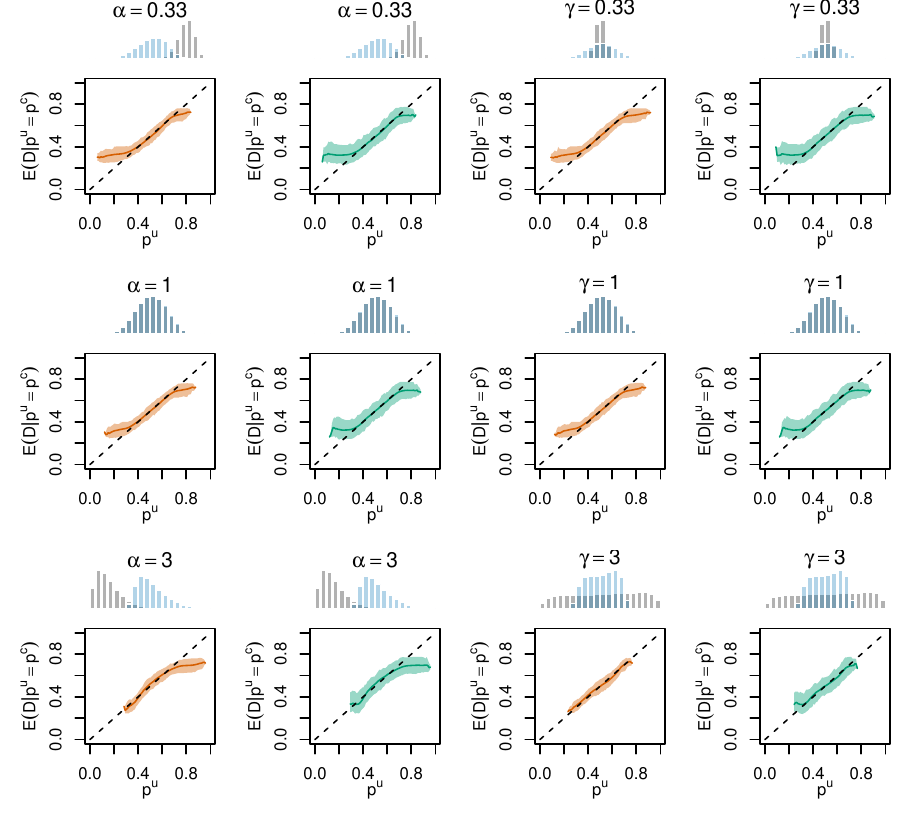}\label{fig-recalib-simuls-calib-curve-locfit-platt}
}\qquad
\subfloat[\textcolor{bleuClairTOL}{Isotonic Regression}.]{
    \includegraphics[width = \figwidth]{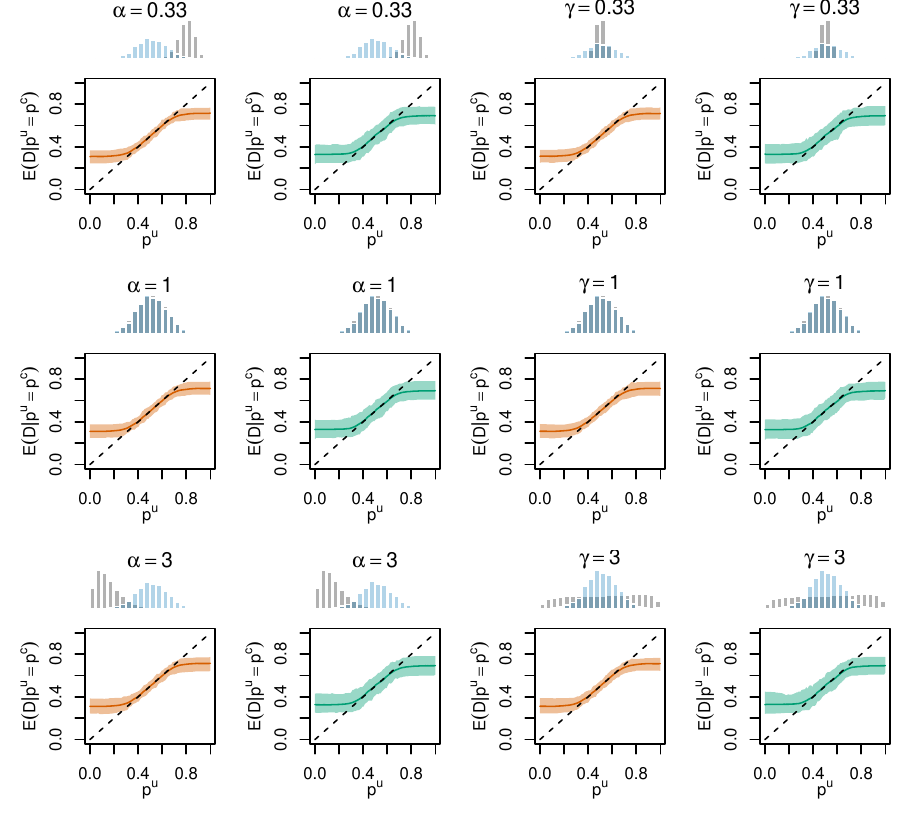}\label{fig-recalib-simuls-calib-curve-locfit-isotonic}
}\\
\subfloat[\textcolor{sableTOL}{Beta Calibration}.]{
    \includegraphics[width = \figwidth]{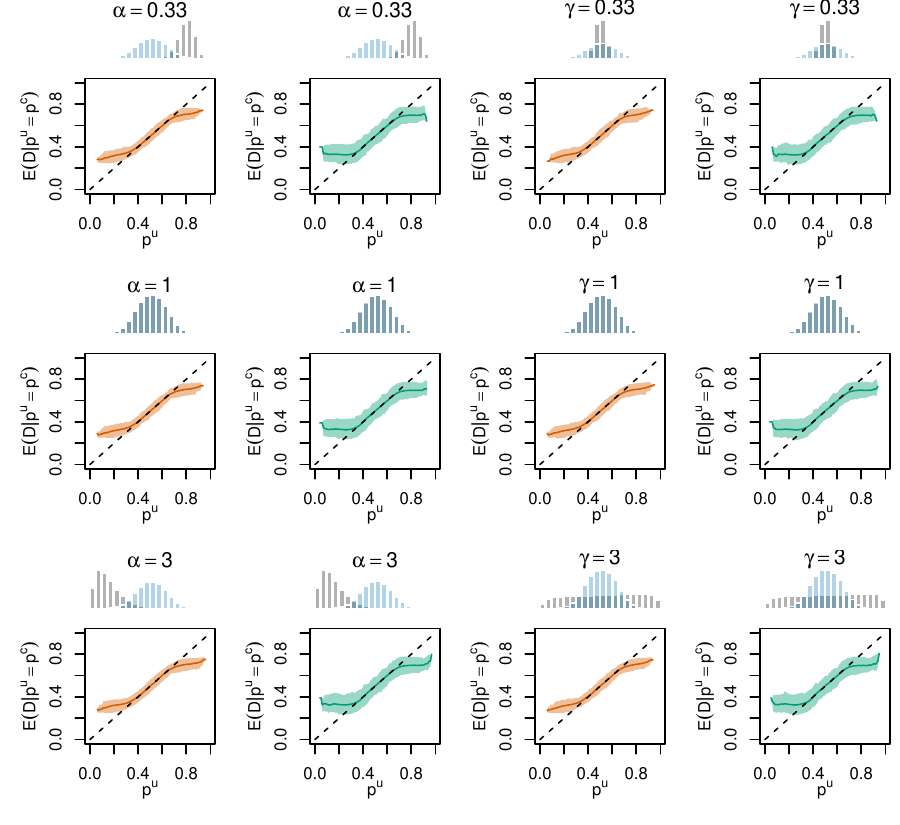}\label{fig-recalib-simuls-calib-curve-locfit-beta}
}
\caption{Calibration Curves Calculated with \textbf{Scores Recalibrated} Using \textcolor{vertClairTOL}{\textbf{Platt Scaling (Panel a)}}, \textcolor{bleuClairTOL}{\textbf{Isotonic Regression (Panel b)}}, or \textcolor{sableTOL}{\textbf{Beta Calibration (Panel c)}}. The curves are obtained with \textbf{local regression}, for the \textcolor{wongOrange}{calibration set (orange)} and for the \textcolor{wongGreen}{test set (green)} for varying values of $\alpha$ and $\gamma$. The curves are the average values obtained on 200 replications of the simulations, the bands correspond to 95\% bootstrap interval. The histogram on top of each graph show the distribution of the \textcolor{colUncalibrated}{uncalibrated scores}, and that of the \textcolor{colRecalibrated}{calibrated scores}.}
\label{fig-recalib-simuls-calib-curve-locfit-platt-isotonic-beta}
\end{figure}

\begin{figure}[ht]%
\centering
\subfloat[\textcolor{parmeTOL}{Degree 0}.]{
    \includegraphics[width = \figwidth]{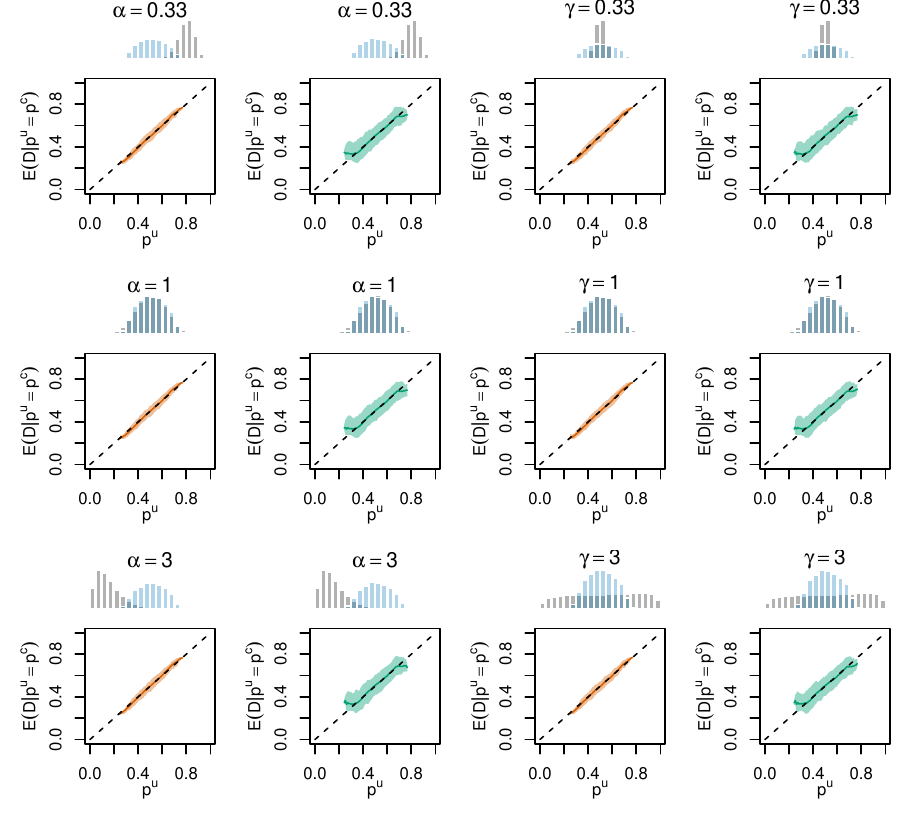}\label{fig-recalib-simuls-calib-curve-locfit-locfit-0}
}\qquad
\subfloat[\textcolor{magentaTOL}{Degree 1}.]{
    \includegraphics[width = \figwidth]{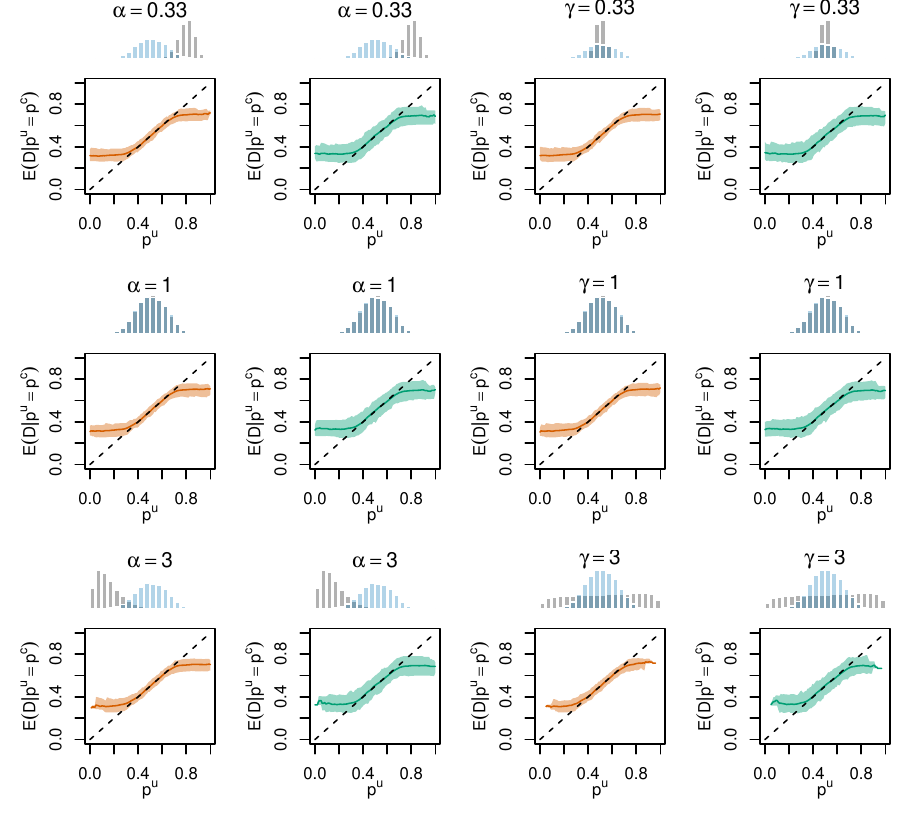}\label{fig-recalib-simuls-calib-curve-locfit-locfit-1}
}\\
\subfloat[\textcolor{roseTOL}{Degree 2}.]{
    \includegraphics[width = \figwidth]{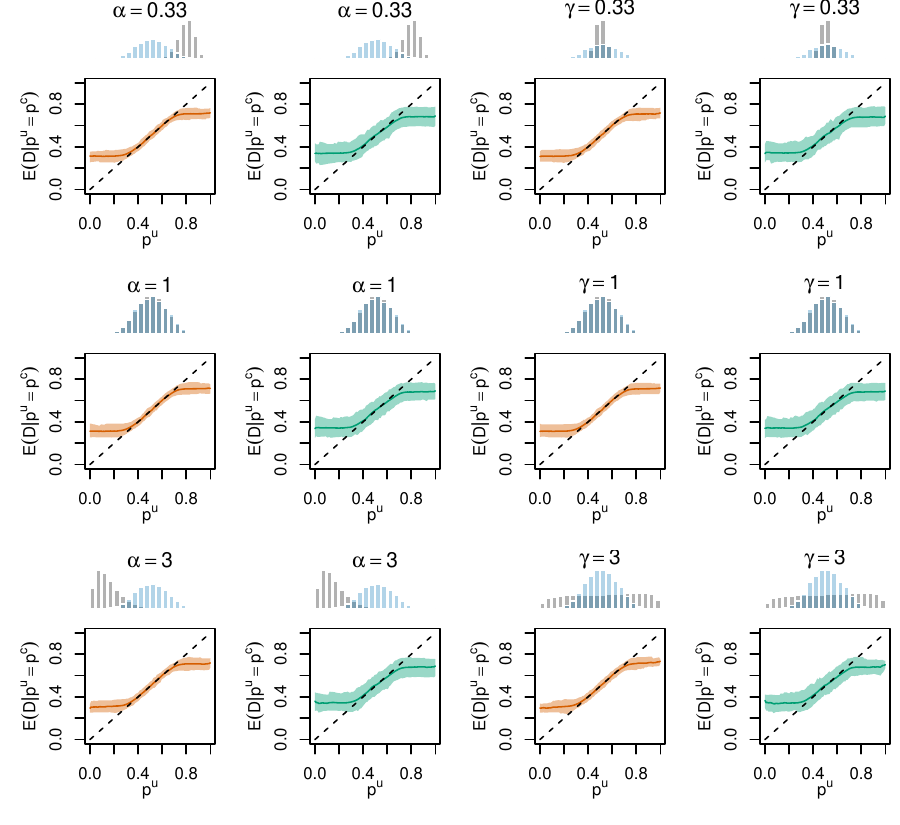}\label{fig-recalib-simuls-calib-curve-locfit-locfit-2}
}
\caption{Calibration Curves Calculated with \textbf{Scores Recalibrated} Using \textbf{Local Regression} with \textcolor{parmeTOL}{Degree 0 (Panel a)}, \textcolor{magentaTOL}{Degree 1 (Panel b)}, or with \textcolor{roseTOL}{Degree 2 (Panel c)}. The curves are obtained with \textbf{local regression}, for the \textcolor{wongOrange}{calibration set (orange)} and for the \textcolor{wongGreen}{test set (green)} for varying values of $\alpha$ and $\gamma$. The curves are the average values obtained on 200 replications of the simulations, the bands correspond to 95\% bootstrap intervals. The histogram on top of each graph show the distribution of the \textcolor{colUncalibrated}{uncalibrated scores}, and that of the \textcolor{colRecalibrated}{calibrated scores}.}
\label{fig-recalib-simuls-calib-curve-locfit-locfit}
\end{figure}

\clearpage
\section{Calibration of a Random Forest}\label{sec-appendix-rf}

We performed a grid search to select the best set of hyperparameters to train two types of forests on a train set with 50\% of the observations: regression forests, and classification forest. For each type of forest, we varied the following hyperparameters:
\begin{itemize}
    \item \texttt{ntree}, the number of trees: 100, 300, 500
    \item \texttt{mtry}, the number of variables to consider for a split: $1, 2, \ldots, 12$, where $12$ represents half the number of features
    \item \texttt{nodesize}, the minimum size in terminal nodes: 5, 10, 15, 20.
\end{itemize}

In total, this corresponds to training 144 regression forests and 144 classification forests. For the regression forest, we computed the out-of-bag MSE, whereas for the classification forest we computed the out-of-bag error rate (computed as the number of incorrectly classified observations over the total number of observations). The best set of hyperparameters was selected as the one which minimizes the out-of-bag criterion. For the regression forest, the best set of hyperparameters turned out to be \texttt{ntree=500}, \texttt{mtry=5}, and \texttt{nodesize=5}. For the classification forest, the best set was \texttt{ntree=300}, \texttt{mtry=5}, and \texttt{nodesize=5}.

After obtaining these hyperparameters, we split the remaining 50\% of observations into a \textcolor{wongOrange}{calibration} and a \textcolor{wongGreen}{test} sets of equal size. This split was performed randomly over 200 different replications. In each replication, we trained a recalibrator in the calibration set, using the predicted scores from the forest. As with the simulated data, we considered the following recalibration techniques: \textcolor{vertClairTOL}{Platt scaling}, \textcolor{bleuClairTOL}{Isotonic regression}, \textcolor{sableTOL}{Beta calibration}, and Local regression with varying degrees: \textcolor{parmeTOL}{degree 0}, \textcolor{magentaTOL}{degree 1}, and \textcolor{roseTOL}{degree 2}. We then computed goodness-of-fit metrics and calibration metrics on the \textcolor{wongOrange}{calibration} and on the \textcolor{wongGreen}{test} set, for each forest and each replication, using either the \textcolor{vertTOL}{uncalibrated scores $\hat{s}(\mathbf{x})$ (denoted as $p_u$)} or on the scores recalibrated with the different methods. 

\subsection{Metrics}

The different metrics are shown in Figure~\ref{recalib-rf-tuned-metrics-gof} 
for goodness-of-fit, and in Figure~\ref{recalib-rf-tuned-metrics-calib} for calibration. Both Figures complement Figure~\ref{fig-recalib-rf-tuned-auc-lsc} from the main text.

\begin{figure}[ht]
\centering
\includegraphics[width = .6\linewidth]{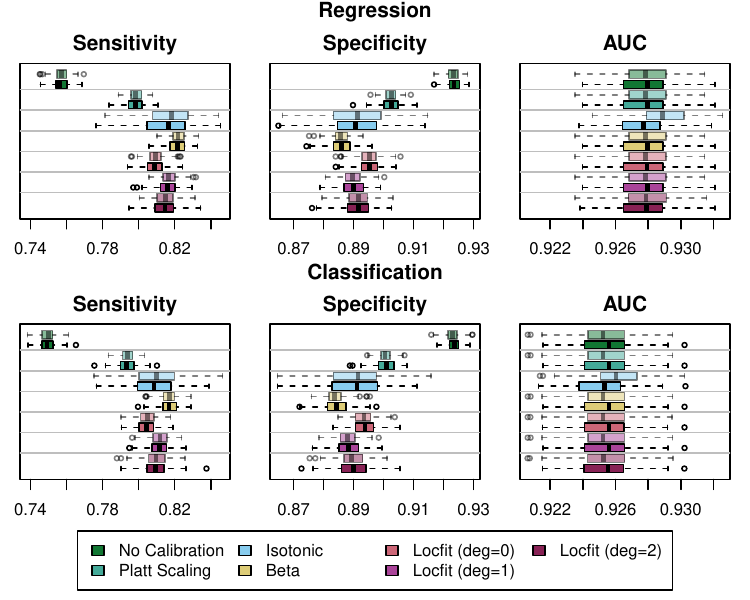}
\caption{Goodness-of-fit Metrics Computed on 200 Replications, for the Regression Random Forest (top) and for the Classification Random Forest (bottom), on the Calibration (transparent colors) and on the Test Set (full colors). A probability threshold of $\tau=0.5$ was used to compute the sensivity and the specificity.}\label{recalib-rf-tuned-metrics-gof}
\end{figure}

\begin{figure}[ht]
\centering
\includegraphics[width = .6\linewidth]{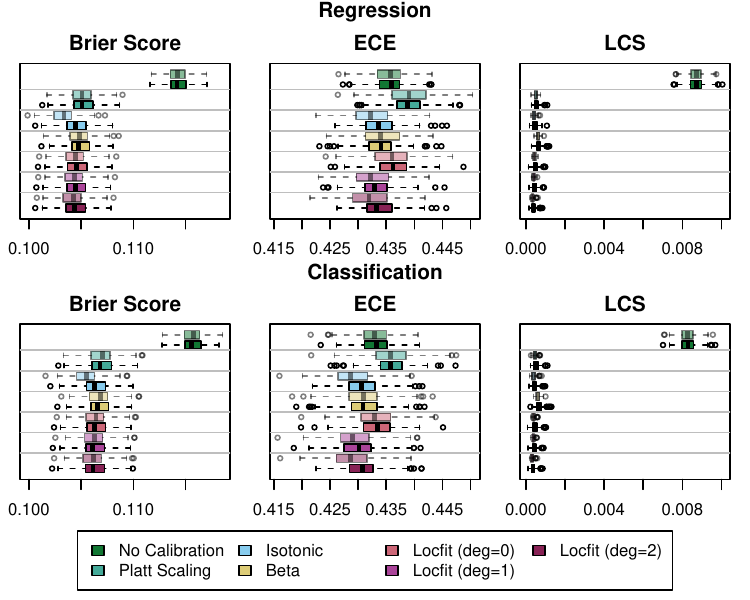}
\caption{Calibration Metrics Computed on 200 Replications, for the Regression Random Forest (top) and for the Classification Random Forest (bottom), on the Calibration (transparent colors) and on the Test Set (full colors).}\label{recalib-rf-tuned-metrics-calib}
\end{figure}

\begin{figure}[ht]%
\centering
\includegraphics[width = .6\textwidth]{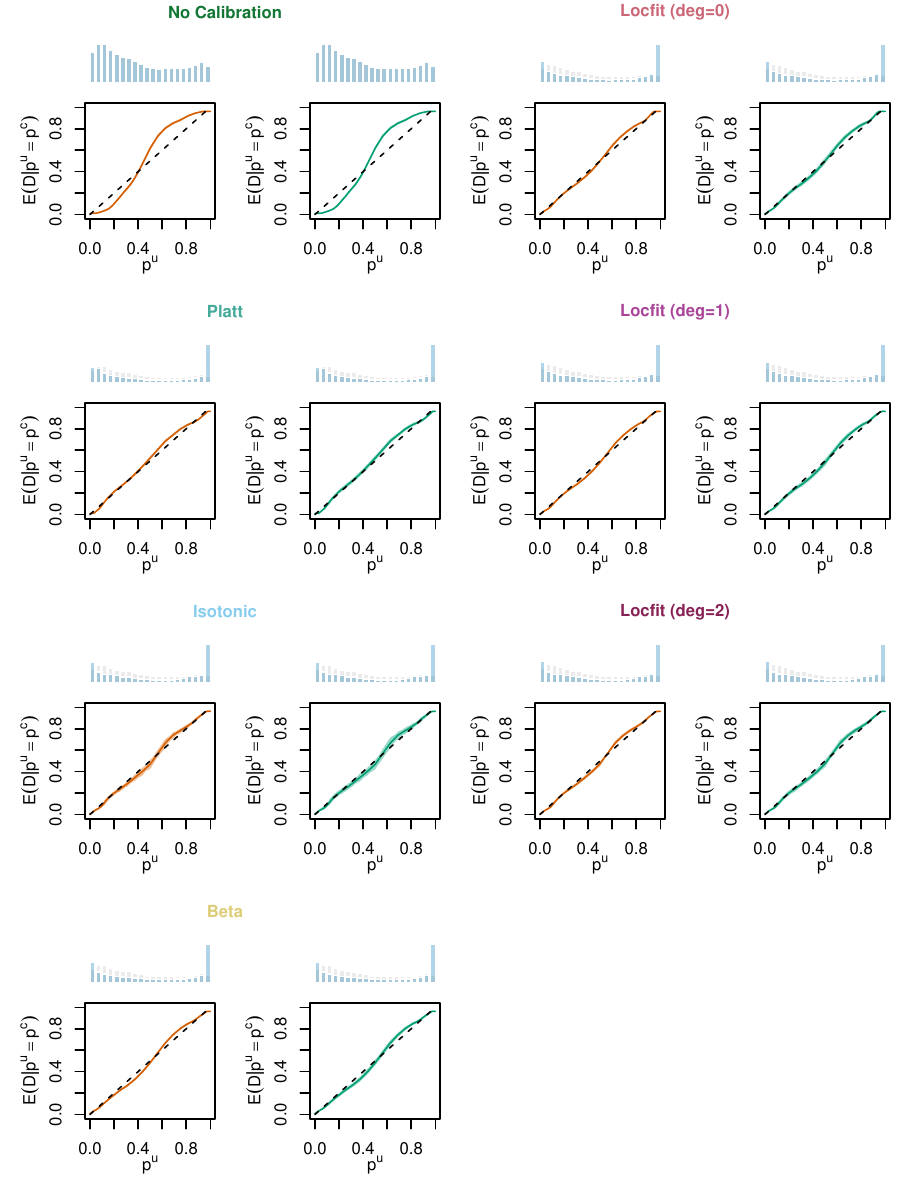}
\caption{Calibration Curves Obtained with \textbf{Local Regression}, for the \textbf{Regression} Random Forest, for the \textcolor{wongOrange}{Calibration Set} and for the \textcolor{wongGreen}{Test Set}. The curves are the averages values obtained on 200 different splits of the calibration and test datasets, and the color bands are the 95\% bootstrap confidence intervals. The histogram on top of each graph show the distribution of the \textcolor{colUncalibrated}{uncalibrated scores}, and that of the \textcolor{colRecalibrated}{calibrated scores}.}
\label{fig-recalib-rf-locfit-curves-reg}
\end{figure}

\begin{figure}[ht]%
\centering
\includegraphics[width = .6\textwidth]{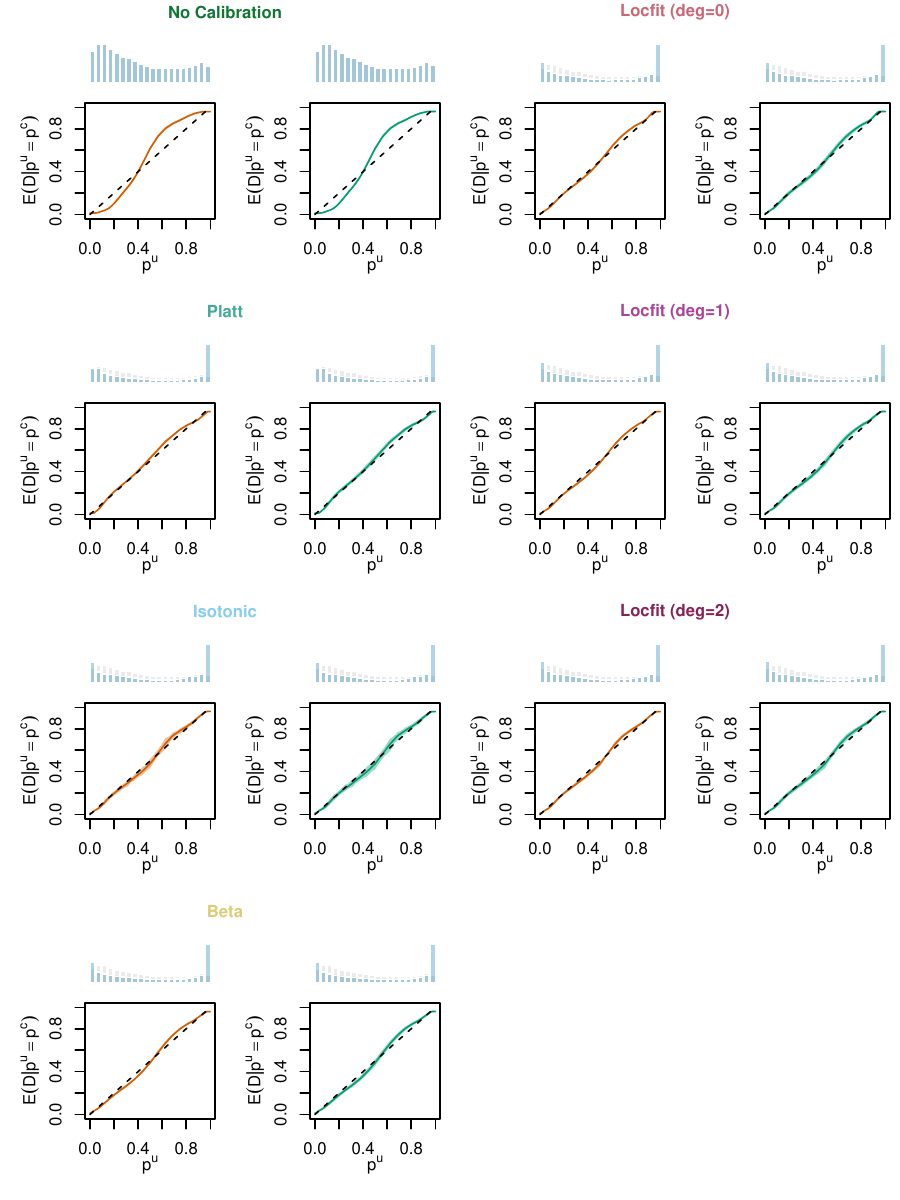}
\caption{Calibration Curves Obtained with \textbf{Local Regression}, for the \textbf{Classification} Random Forest, for the \textcolor{wongOrange}{Calibration Set} and for the \textcolor{wongGreen}{Test Set}. The curves are the averages values obtained on 200 different splits of the calibration and test datasets, and the color bands are the 95\% bootstrap confidence intervals. The histogram on top of each graph show the distribution of the \textcolor{colUncalibrated}{uncalibrated scores}, and that of the \textcolor{colRecalibrated}{calibrated scores}.}
\label{fig-recalib-rf-locfit-curves-classif}
\end{figure}

\end{document}